\documentclass[final,journal]{IEEEtran}

\usepackage[cmex10]{amsmath}
\usepackage{graphicx}
\usepackage{times}
\usepackage{dsfont}
\usepackage[latin1]{inputenc}
\usepackage{fixltx2e}
\usepackage{amsfonts}
\usepackage{dsfont}
\usepackage{amsthm}
\usepackage{url}
\usepackage{booktabs}
\usepackage{units}
\usepackage{cite}
\usepackage{array}
\usepackage{mdwmath}
\usepackage{mdwtab}
\usepackage{eqparbox}
\usepackage[nearskip,caption=false,font=footnotesize]{subfig}
\usepackage{flushend}

\usepackage{algorithm2e}

\usepackage{amssymb}
\usepackage{wasysym}

\usepackage{multicol}
\usepackage{multirow}
\usepackage{rotating}

\usepackage{tabularx}

\usepackage{color}

\newtheorem{proposition}{Proposition}

\begin{document}
%
\title{High Level Pattern Classification \\via Tourist Walks in Networks}

\author{Thiago~Christiano~Silva
        and~Liang~Zhao,~\IEEEmembership{Senior~Member,~IEEE}
\IEEEcompsocitemizethanks{\IEEEcompsocthanksitem Thiago Christiano Silva and Liang Zhao
are with the Department of Computer Sciences, Institute of
Mathematics and Computer Science (ICMC), University of São Paulo
(USP), Av. Trabalhador São-carlense, 400, 13560-970, São Carlos,
SP, Brazil.\protect\\
E-mail: \{thiagoch, zhao\}@icmc.usp.br.}
\thanks{}}

\markboth{IEEE Transactions on Neural Networks and Learning Systems}%
{Thiago Christiano Silva and Liang Zhao: High Level Classification via Tourist Walks}

\IEEEcompsoctitleabstractindextext{%
\begin{abstract}

Complex networks refer to large-scale graphs with nontrivial connection patterns. The salient and interesting features that the complex network study offer in comparison to graph theory are the emphasis on the dynamical properties of the networks and the ability of inherently uncovering pattern formation of the vertices. In this paper, we present a hybrid data classification technique combining a low level and a high level classifier. The low level term can be equipped with any traditional classification techniques, which realize the classification task considering only physical features (e.g., geometrical or statistical features) of the input data. On the other hand, the high level term has the ability of detecting data patterns with semantic meanings. In this way, the classification is realized by means of the extraction of the underlying network's features constructed from the input data. As a result, the high level classification process measures the compliance of the test instances with the pattern formation of the training data. Out of various high level perspectives that can be utilized to capture semantic meaning, we utilize the dynamical features that are generated from a tourist walker in a networked environment. Specifically, a weighted combination of transient and cycle lengths generated by the tourist walk is employed for that end. Furthermore, we show that the proposed technique is able to capture the organizational and complex features of the class component from a local to global fashion in a natural and intuitive way by altering the memory size of the tourist walk. Still in this work, we uncover the existence of a critical memory length, we say \emph{complex saturation}, where any values larger than this critical point make no change in the transient and cycle lengths of the network component. Interestingly, our study shows that the proposed technique is able to further improve the already optimized performance of traditional classification techniques. Finally, we apply the proposed technique to the recognition of handwritten digit images and promising results have been obtained.
\end{abstract}

}

\def \sizeOfFigure {0.22}

\maketitle

\IEEEdisplaynotcompsoctitleabstractindextext

\IEEEpeerreviewmaketitle

\section{Introduction}

\IEEEPARstart{S}{upervised} data classification aims at generating a map from the input data to the corresponding desired output, for a given training set. The constructed map, called a classifier, is used to predict new input instances. Many supervised data classification techniques have been developed \cite{Vapnik1998, Duda2001, Bishop2006}, such as $k$-nearest neighbors, Bayesian classifiers, neural networks, decision trees, committee machines, and so on. In essence, all these techniques train and, consequently, classify unlabeled data items according to the physical features (e.g., distance, similarity or distribution) of the input data. These techniques that predict class labels using only physical features are called \emph{low level classification} techniques \cite{Silva2012HighLevel}.

Usually, the data items are not isolated points in the attribute space, but instead tend to form certain patterns. For example, in Fig. \ref{fig:Motivation-2-Classes}, the two test instances represented by the triangle-shaped are most probably to be classified as pertaining to the square-shaped class if only physical features, such as distances among data instances, are considered. On the other hand, if we take into account the relationships among the data, we intuitively classify the triangle-shaped items as members of the circular-shaped class, since a clear pattern (lozenge) is formed. The human (animal) brain performs both low and high orders of learning and it has facility to identify patterns according to the semantic meaning of the input data. However, this kind of task, in general, is still hard to be performed by computers. Supervised data classification by considering not only physical attributes but also pattern formation is referred to as \emph{high level classification} \cite{Silva2012HighLevel}.

\begin{figure} [!htb]
    \centering
    \includegraphics[scale = \sizeOfFigure]{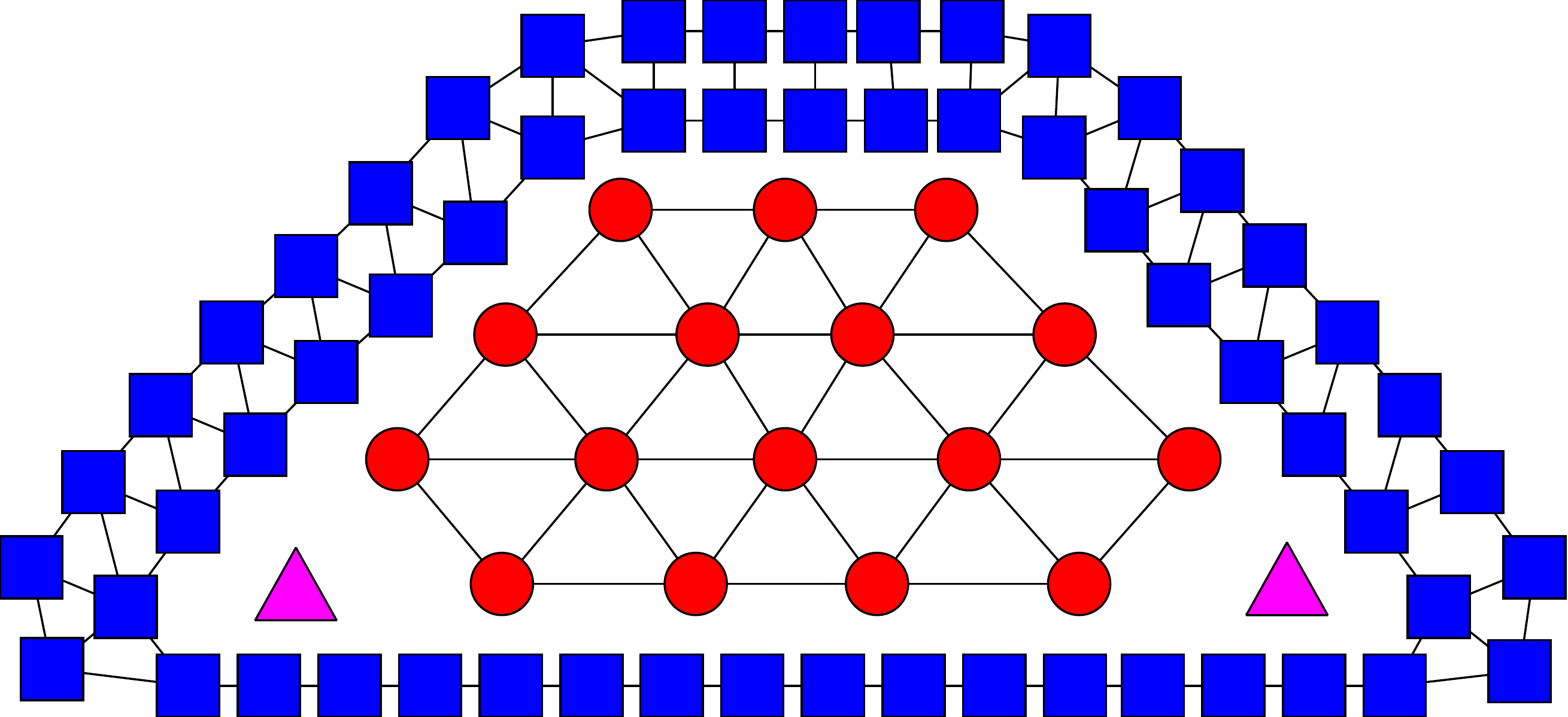}%
    \caption{A simple example of a supervised data classification task where two clear patterns are formed: the highly organized circular-shaped class and a rather dense square-shaped class. The goal is to classify the two triangle-shaped data items.}%
    \label{fig:Motivation-2-Classes}
\end{figure}

Traditional classification techniques often share the same vision: dividing the data space into sub-spaces, each of which representing a class. They are short in reproducing complex-formed or twisted classes. On the other hand, the salient feature of the proposed technique is that it provides two really distinct orthogonal visions: low level and high level visions for data classification. The former views the data's classes from the perspective of their physical features, while the latter captures pattern formations of the data, which, in turn, permits the classifier to reproduce complex-formed and (or) twisted classes, i.e., a test instance can be put into a class if it conforms with the pattern formed by that class no matter how far it is from the center of that class. In this sense, strongly related techniques are co-training \cite{Blum1998} and tri-training \cite{Zhi-Hua2005}, which attempt to consider the cooperation of various
classification techniques (ensemble), each focusing on a theoretically different ``vision'' of the data. However, the involved classification techniques in co-training or tri-training are really only making statistically independent decisions, which still share, in essence, the same data vision. Another difficulty in these techniques resides in building these statistically independent visions. Since it is supposed that the data is generated from an unknown distribution, it is often hard to achieve such task. In our proposed work, the uncorrelated visions are naturally captured by the classifiers themselves, not by changing the content of the data items, but by ``looking'' at the data relationships in a different way (low and high levels).

Following the literature stream on such matter, there are several kinds of works related to high level classification, such as the Semantic Web \cite{Berners-LeeEtAl2001, ShadboltEtAl2006, FeigenbaumEtAl2007}, which uses ontologies to describe the semantics of the data, statistical relational learning, which realizes collective inference \cite{ZhangPD06, Macskassy2007, ZhuYCG07, Gallagher2008, Zhang2008} or graph-based semisupervised learning \cite{Zhu2005, Chapelle2006}, and contextual classification techniques \cite{DonaldsonToussaint1970, BinaghiEtAl2003, TuiaEtAl2003, TianEtAl2006, LuWeng2007, WilliamsEtAl2007, Micheli2009}, which consider the spatial relationships between the individual pixels and the local and global configurations of neighboring pixels in an image for assigning classes.

From the viewpoint of high level classification, all the above mentioned approaches are quite restricted either to the types of semantic features to be extracted, such as Semantic Web, or to the types of data, such as the contextual classification, which is devoted to considering spatial relationship among pixels in image processing. To our knowledge, it is still lacking an explicit and general scheme to deal with high level classification in the literature, which is quite desirable for many applications, such as invariant pattern recognition. The current paper presents an endower to this direction.

As mentioned, low level classification usually presents difficulty in identifying the complex relationships among the data items. Consequently, these techniques are not suitable for uncovering semantically meaningful patterns formed by the data. This is because the data patterns are often not encountered with a fixed shape or distribution, instead, they are frequently determined by the local and/or global interactions among the data items. It is well know that the network representation can capture arbitrary levels of relationships or interactions of the input data. For this reason, we here show how the topological properties of the input data can help in identifying the pattern formation and, consequently, can be used for general high level classification. In this case, the topological properties are revealed by \emph{tourist walks}. A tourist walk can be defined as follows: Given a set of cities, each time the tourist (walker) goes to the nearest city that has not been visited in the past $\mu$ time steps \cite{Lima2001}. It has been shown that tourist walk is useful for data clustering \cite{Monica2006} and image processing \cite{Backes2010}. However, all these kinds of works are realized in regular lattices. Here, we study tourist walk in networks and we show that it has the ability of capturing the topological properties of the underlying network in a local to global fashion. Moreover, the tourist walks approach applied to a networked environment is a relatively new approach taken here. Additionally, its utilization for discovering patterns in a network is a totally novel scheme in the
literature.

In this paper, we propose a technique that combines the low level and the high level supervised data classifications. The idea of this paper is built upon the general framework proposed by \cite{Silva2012HighLevel}. The low level classification can be implemented by any traditional classification technique, while the high level classification exploits the complex topological properties of the underlying network constructed from the input data.
In the original work introduced in \cite{Silva2012HighLevel}, the high level classification problem is treated using three existing network measures in a combined way (assortativity, clustering coefficient, and average degree). Thus, a serious open problem is how one may choose other network measures in an intuitive way and also how to define the inference weight that is given for each of them. In this paper, a novel measure for high level classification is introduced by using tourist walks in networks. Since they are a deterministic dynamical process, the local and global information of the underlying network can be detected by the exclusive use of such measure. Moreover, the use of tourist walks presents some nontrivial advantages over the previous approach, such as:

\begin{itemize}
    \item It is able to capture the organizational and complex features of the class component from a local to global fashion in a natural and intuitive way. For example, when the memory window of the tourist is low, it is able to extract local features of the graph component. As we increase the memory window, the dynamics of the walk compels the tourist to venture far away from its starting vertex. Hence, it is able to capture more global features of the graph component;

    \item It occurs that the tourist walk method presents a critical memory length, where any values larger than this critical point make no change in the transient and cycle lengths of the graph component. This is an interesting phenomenon, which is observed when the memory length reaches a sufficient high value. We say that, when this happen, the walks have reached the ``complexity saturation'' of the class component.  In this occasion, the global topological and organizational features of the network
        are said to be completely characterized in the sense of tourist walks;

    \item In view of the intuitive dynamical properties displayed by a tourist walk, one can avoid the weight assignment among various network measures, which is a problem when static network measures are used, as occurs in \cite{Silva2012HighLevel}. This is because static measures provide partial or instant visions of the underlying network. Thus, one must artificially combine them to get a global vision.

\end{itemize}

Still in this paper, we show how the proposed technique can be used to solve general invariant pattern recognition problems \cite{Hamsici2009, Park2010, Serre2007}, particularly when the pattern variances are nonlinear and there is not a closed form to describe the invariance.

The remainder of the paper is organized as follows. A detailed overview of the tourist walks is supplied in Section \ref{relevant-background}. The proposed model is defined in Section \ref{Model-Description-supervised}. Computer simulations are performed on synthetic and real-world data sets in Section \ref{Computer-Simulations}. In Section \ref{Simulation-Real}, we apply the proposed technique to manual digits recognition problem. Finally, Section \ref{sec:Conclusions} concludes the paper.

\section{Relevant Background: Tourist Walks}
\label{relevant-background}

A tourist walk can be can conceptualized as a walker (tourist) aiming at visiting sites (data items) in a $d$-dimensional map, representing the data set. At each discrete time step, the tourist follows a simple deterministic rule: it visits the nearest site which has not been visited in the previous $\mu$ steps. In other words, the walker performs partially self-avoiding deterministic walks over the data set, where the self-avoiding factor is limited to the memory window $\mu - 1$. This quantity can be understood as a repulsive force emanating from the sites in this memory window, which prevents the walker from visiting them in this interval (refractory time). Therefore, it is prohibited that a trajectory to intersect itself inside this memory window. In spite of being a simple rule, it has been shown that this movement dynamic possesses complex behavior when $\mu > 1$ \cite{Lima2001}.

The tourist's behavior heavily depends on the data set's configuration and the starting site. In computational terms, the tourist's movements are entirely realized by means of a neighborhood table. This table is constructed by ordering all the data items in relation to a specific site. This procedure is performed for every site of the data set.

Each tourist walk can be decomposed in two terms: (i) the initial \emph{transient part} of length $t$ and (ii) a \emph{cycle} (attractor) with period $c$. Figure \ref{img:tourist-walk-schematic} shows an illustration of a tourist walk with $\mu = 1$, i.e., the walker always goes to the nearest neighbor. In this case, one can see that the transient length is $t=3$ and the cycle length $c=6$.

Considering the attractor or cycle period as a walk section that begins and ends at the same site of the data set may lead one to think that, once the tourist visits a specific site, a new visit to it would configure an attractor. However, during a walk, a site may be re-visited without configuring an attractor. For instance, if we had chosen a $\mu = 6$ for the walk in Fig. \ref{img:tourist-walk-schematic}, the re-visit that the tourist performs on the site $4$ would have not configured an attractor, since the site $5$ would still be forbidden site to be visited again; hence, the tourist would be compelled to visit another site. This characteristic enables sophisticated trajectories over the data set, at cost of also increasing the difficulty of detecting an attractor.

A note that is worth pointing out is that, in the majority of the works related to these walks \cite{Lima2001,Stanley2001,Kinouchi2002}, the tourist may visit any other site other than the ones contained in its memory window. As $\mu$ increases, there is a significant chance that the walker will begin performing large jumps in the data set, since the neighborhood is most likely to be already visited in its entirety within the time frame $\mu$. As we will show in this work, in the context of classification, this is an undesirable characteristic that can be simply avoided by using a graph representation of the input data. In this way, the walker is only permitted to visit vertices, represented now by the sites, that are in its connected neighborhood (link). With this modified mechanism, it is most probable that, for large values of $\mu$, depending on the network configuration, the walker will get trapped within a vertex, not being able to further visit other vertices of the neighborhood. In this scenario, we say that the walk only had a transient part and the cycle period is null ($c=0$). Therefore, the tourist walks approach applied to a networked environment is a relatively new approach taken here. Additionally, its utilization for discovering patterns in a network is a totally novel scheme in the literature.

\begin{figure}
  \centering
  \includegraphics[scale=\sizeOfFigure]{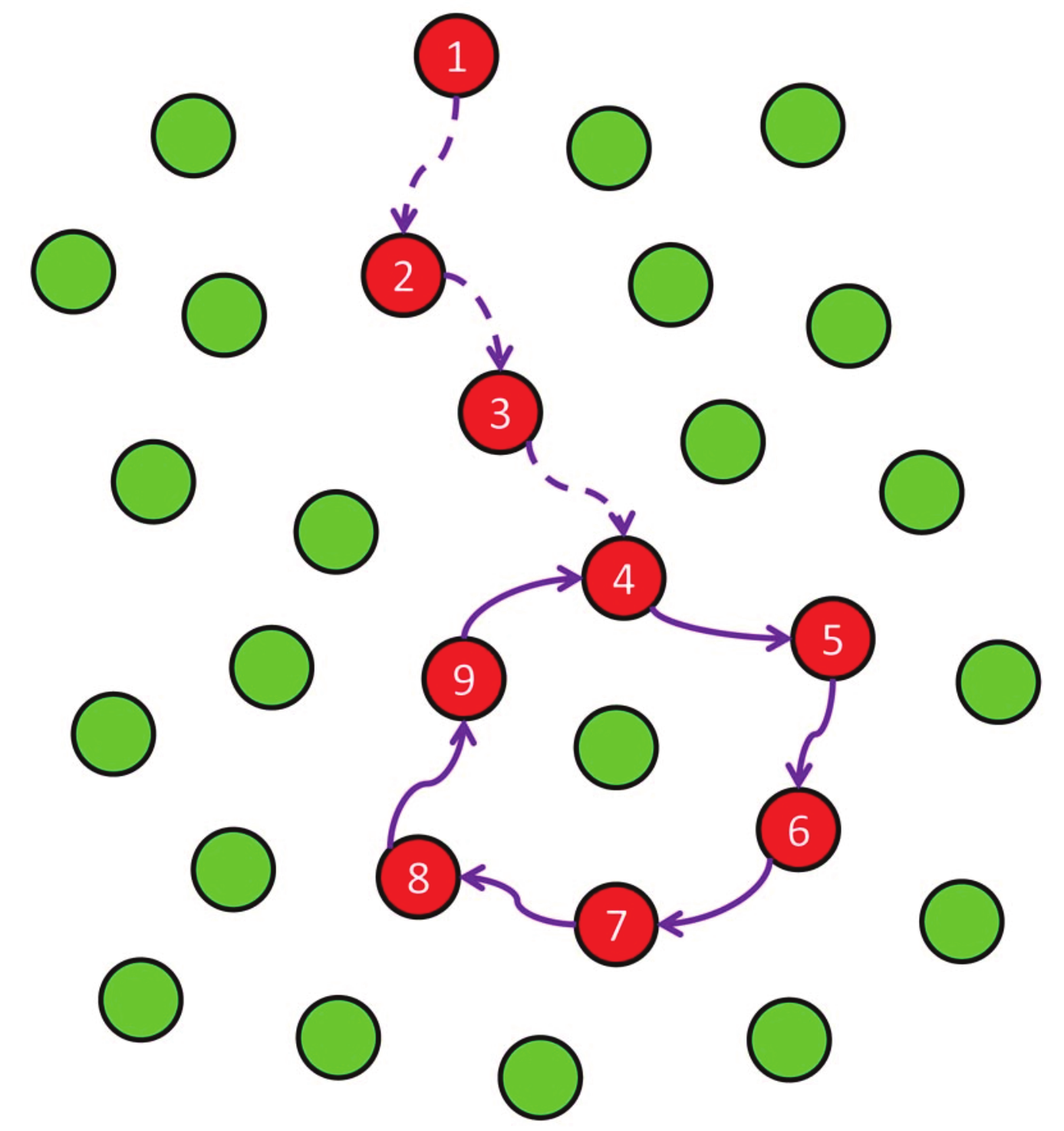}
  \caption{Illustration of a tourist walk with $\mu = 1$. The red and green dots represent visited and unvisited sites, respectively.
  The dashed lines indicate the transient part of the walk, whereas the continuous lines, the attractor of the walk.}
  \label{img:tourist-walk-schematic}
\end{figure}

\section{Model Description}
\label{Model-Description-supervised}

In this section, a useful set of notations is presented, along with the premises of the underlying hybrid classification framework. Next, we give a quick overview on the particularities of the original hybrid classification framework \cite{Silva2012HighLevel}. Finally, the high level classification based on tourist walks is formally introduced, as well as an overview of its algorithm.

\subsection{Notations and Premises}
\label{sec:network-formation-supervised}

The hybrid classifier is designed to work in a \emph{supervised learning} environment. In the following, some mathematical notations and premises are discussed.

Consider that $\mathcal{X}_{\mathrm{training}} = \{(x_{1},y_{1}),\hdots,(x_{l},y_{l})\} \subset \mathcal{X} \times \mathcal{L}$ denotes the training set, which is composed of $l$ labeled training instances. Each training instance, $x_{i} \in \mathcal{X}$, is given a discrete label or target $y_{i} \in \mathcal{L}$. Furthermore, each training instance is described by a $d$-dimensional vector, i.e., $x_{i} = (f_{1},\hdots,f_{d})$, where each entry symbolizes a feature or descriptor of that item.

The goal here is to construct a hypothesis, in a way that the classifier maps $x \mapsto y$. Commonly, the constructed classifier is checked with regard to its prediction power by submitting it to a test set $\mathcal{X}_{\mathrm{test}} = \{x_{l+1},\hdots,x_{l+u}\}$, in which labels are not provided. In this case, each data item is called test instance.
For an unbiased learning, the training and test sets must be disjoint, i.e.,  $\mathcal{X}_{\mathrm{training}} \cap \mathcal{X}_{\mathrm{test}} = \emptyset$.

\subsection{Overview of the Hybrid Classification Framework}

In this section, we review the hybrid classification framework \cite{Silva2012HighLevel}. Specifically, in the following sections, the particularities of the \emph{training and classification phases}, and the general definition of the classification scheme are discussed.

\subsubsection{Training Phase}

In this phase, the data in the training set are mapped into a graph $\mathcal{G}$ using a network formation technique $g: \mathcal{X}_{\mathrm{training}} \mapsto \mathcal{G} = \langle \mathcal{V}, \mathcal{E} \rangle$, where $\mathcal{V} = \{1,\hdots,V\}$ is the set of vertices and $\mathcal{E}$ is the set of edges. Each vertex in $\mathcal{V}$ represents a training instance in $\mathcal{X}_{\mathrm{training}}$. As it will be described later, the pattern formation of the classes will be extracted by using the complex topological features of this networked representation. Therefore, the network construction is vital for the prediction produced by the high level classifier.

In this stage, we first construct a network component for each class of the vector-based training set. The
strategy to create edges depends on the type of region in which each vertex is. When the region is sparse, we utilize a $k$-nearest neighbor ($k$-NN) approach, whereas, when it is dense, we employ the $\epsilon$-radius technique. While the $k$-NN sets up an edge between the $k$ most similar vertices and the reference vertex, the $\epsilon$-radius method creates a link to whichever vertex that is within a predefined distance with radius $\epsilon$. The way that we classify a region as dense or sparse is by checking whether, within a circular region $\epsilon$ centered in the reference vertex, there are more than $k$ vertices of the same class as the reference vertex.
If so, the region is classified as dense and the $\epsilon$-radius is selected; otherwise, the $k$-NN is elected. During the network formation, we are only permitted to create edges between vertices of the same class. The reason why we employ a $k$-NN in sparse region is to prevent the appearance of multiple graph components representing the same class. Therefore, we expect that this process will form isolated class components, in a way that each class is guaranteed to have a single and unique component representing it.

\subsubsection{Classification Phase}

In the classification phase, the unlabeled data items in the $\mathcal{X}_{\mathrm{test}}$ are presented to the classifier one by one. In contrast to the training phase, the class labels of the test instances are unknown. In view of that, we keep utilizing the edge formation strategy previously introduced but with slight changes: now, we do not consider the labels of the neighboring vertices. All the other technicalities remain the same. This prevents test instances from becoming singleton vertices, provided that $k > 0$.

With respect to the high order of learning, once the data item is inserted, each class analyzes, in isolation, the impact of the insertion of this data item on its respective class component by using a number of complex topological features. In the proposed high level model, each class retains an isolated graph component. Each of these components calculates the changes that occur in its pattern formation with the insertion of this test instance. If slight or no changes occur, then it is said that the test instance is in compliance with that class pattern. As a result, the high level classifier yields a great membership value for that test instance on that class. Conversely, if these changes dramatically modify the class pattern, then the high level classifier produces a small membership value on that class. These changes are quantified via network measures, each of which numerically translating the organization of the component from a local to global fashion. As we will see, we will extract information generated from the dynamical process of a tourist walker in a networked environment in an intuitive manner, such as to capture the pattern formation of the network in a local to global basis.

\subsubsection{General Hybrid Classification Framework}

Tthe hybrid classification framework $F$ consists of a convex combination of two orthogonal terms, where each of which renders an uncorrelated vision about the data items, as follows:

\begin{enumerate}
    \item[i.] A low level classifier, for instance, a decision tree, SVM, or a $k$-NN classifier. The vision that it stresses is the physical similarities among the data;

    \item[ii.] A high level classifier, which is responsible for classifying a test instance according to its organizational or semantic meaning with the data. The vision that it values most is the pattern compliance of new test instances with the existing structure built up in the training process.
\end{enumerate}

Mathematically, the membership of the test instance $x_{i} \in \mathcal{X}_{\mathrm{test}}$ with respect to the class $j \in \mathcal{L}$ yielded by the hybrid framework, here written as $F^{(j)}_{i}$, is given by:

\begin{align}
    F^{(j)}_{i} = (1 - \lambda)L^{(j)}_{i} + \lambda H^{(j)}_{i},
    \label{eq:def-classification}
\end{align}

\noindent where $L^{(j)}_{i}, H^{(j)}_{i} \in [0,1]$ represent the membership of the test instance $x_{i}$ towards class $j$ produced by the low and high level classifiers, respectively, and $\lambda \in [0,1]$ is the \emph{compliance term}, which plays the role of counterbalancing the classification decisions supplied by both low and high level classifiers. Note that, when $\rho = 0$, (\ref{eq:def-classification}) reduces to a common low level classifier.

A test instance $x_{i}$ receives the label of the class $j \in \mathcal{L}$ that maximizes (\ref{eq:def-classification}). Mathematically, the estimated label of $x_{i}$, $\hat{y}_{x_{i}}$, is decided according to the following expression:

\begin{align}
    \hat{y}_{x_{i}} = \underset{j \in \mathcal{L}}{\mathrm{arg} \operatorname{\mathrm{max}\mbox{ }}} F^{(j)}_{i}.
    \label{eq:classification-equation}
\end{align}

Note that the predictions produced by both classifiers are combined via a linear combination to derive the prediction of the high level framework (meta-learning). Once the test instance $x_{i}$ gets classified, it is either discarded or incorporated to the training set with the corresponding predicted label. In the second case, only the edges created between the test instance and the class that it belongs to are maintained. Note that, in any of the two situations, each class is still represented by a single graph component.

\subsection{Deriving the High Level Classification Technique Using Tourist Walks}
\label{Classification-Technique}

Equation (\ref{eq:def-classification}) supplies a general framework for the hybrid classification process, in the sense that various supervised data classification techniques can be brought into play. The first term of (\ref{eq:def-classification}) is rather straightforward to implement, since it can be any traditional classification technique. The literature provides a myriad of supervised data classification techniques. Some of these include graph-based methods, decision trees, SVM and its variations, neural networks, Bayesian learning, among many others. However, to our knowledge, little has been done in the area of classifiers that take into account the patterns or organizational features inherently hidden into the relationships among the data items. Thus, we now proceed to a detailed analysis of the proposed high level classification term $H$.

Motivated by the intrinsic ability of describing topological structures among the data items, we propose a network-based (graph-based) technique for the high level classifier $H$. Specifically, the inference of pattern formation within the data is processed using the generated network. In order to do so, the following structural constraints must be satisfied for any constructed network:

\begin{enumerate}
    \item[i.] Each class is an isolated subgraph (graph component);

    \item[ii.] Each class retains a representative and unique graph component.
\end{enumerate}

As we have drawn attention to, the pattern formation of the data is quantified through a combination of network measures generated by a tourist walker in a networked environment. These measures are chosen in a way to cover relevant high level aspects of the class component. One can conceive these dynamical values as true network measures representing each class component.
Having in mind the basic concepts revolving around tourist walks, the decision output of the high level classifier is given by:

\begin{align}
    H^{(j)}_{i} = \frac{\sum_{\mu=0}^{\mu_{c}}{\left[\alpha_{t}(1 - T_{i}^{(j)}(\mu)) + \alpha_{c}(1 - C_{i}^{(j)}(\mu)) \right]} }{\sum_{g\in \mathcal{L}}{\sum_{\mu=0}^{\mu_{c}}{\left[\alpha_{t}(1 - T_{i}^{(g)}(\mu)) + \alpha_{c}(1 - C_{i}^{(g)}(\mu)) \right]}}}
    \label{eq:def-C-term-high-level}
\end{align}

\noindent where $\mu_{c}$ is a critical value that indicates the maximum memory length of the tourist walks,
$\alpha_{t},\alpha_{c} \in [0,1]$ are user-controllable coefficients that indicate the influence of each network measure in the
process of classification, $T_{i}^{(j)}(\mu)$ and $C_{i}^{(j)}(\mu)$ are functions that depend on the transient and cycle
lengths, respectively, of the tourist walk applied to the $i$th data item with regard to the class $j$. These functions are responsible for providing an estimative whether or not the data item $i$ under analysis possesses the same patterns of component
$j$. The denominator in (\ref{eq:def-C-term-high-level}) has been introduced solely for normalization matters. Indeed, in order to
(\ref{eq:def-C-term-high-level}) to be a valid convex combination of network measures, $\alpha_{t}$ and $\alpha_{c}$ must be chosen such
as to satisfy $\alpha_{t} + \alpha_{c} = 1$.

Regarding $T_{i}^{(j)}(\mu)$ and $C_{i}^{(j)}(\mu)$, they are given by the following expressions:

\begin{align}
    \begin{split}
    T_{i}^{(j)}(\mu) = \Delta t^{(j)}_{i}(\mu) p^{(j)}\\
    C_{i}^{(j)}(\mu) = \Delta c^{(j)}_{i}(\mu) p^{(j)}
    \end{split}
    \label{eq:f-function-def-high-level}
\end{align}

\noindent where $\Delta t^{(j)}_{i}(u), \Delta c^{(j)}_{i}(u) \in [0,1]$ are the variations of the transient and cycle lengths that occur on the component representing class $j$ if $i$ joins it and $p^{(j)} \in [0,1]$ is the proportion of data items pertaining to class $j$. Remembering that each class has a component representing itself, the strategy to check the pattern compliance of a test instance is to examine whether its insertion causes a great variation of the network measures representing the class component. In other words, if there is a small change in the network measures, the test instance is in compliance with all the other data items that comprise that class component, i.e., it follows the same pattern as the original members of that class. On the other hand, if its insertion is responsible for a significant variation of the component's network measures, then probably the test instance may not belong to that class. This is exactly the behavior that (\ref{eq:def-C-term-high-level}) together with (\ref{eq:f-function-def-high-level}) propose, since a small variation of $f(u)$ causes a large membership value output by $H$; and vice versa.

In the following, we explain how to compute $\Delta t^{(j)}_{i}(\mu)$ and $\Delta c^{(j)}_{i}(\mu)$ that appear in (\ref{eq:f-function-def-high-level}). Firstly, we need to numerically quantify the transient and cycle lengths of a component. Since the tourist walks are strongly dependent on the starting vertices, for a fixed $\mu$, we perform tourist walks initiating from each one of the vertices that are members of a class component. The transient and cycle lengths of the $j$th component, $\langle t^{(j)}_{i}(\mu) \rangle$ and $\langle c^{(j)}_{i}(\mu) \rangle$, are simply given by the average transient and cycle lengths of all its vertices, respectively. In order to estimate the variation of the component's network measures, consider that $x_{i} \in \mathcal{X}_{\mathrm{test}}$ is a test instance. In relation to an arbitrary class $j$, we virtually insert $x_{i}$ into component $j$ using the network formation technique that we have seen, and recalculate the new average transient and cycle lengths of this component. We denote these new values as $\langle {t'}^{(j)}_{i}(\mu) \rangle$ and $\langle {c'}^{(j)}_{i}(\mu) \rangle$, respectively. This procedure is performed for all classes $j \in \mathcal{L}$. It may occur that some classes $u \in \mathcal{L}$ will not share any connections with the test instance $x_{i}$. Using this approach, $\langle t^{(k)}_{i}(\mu) \rangle = \langle {t'}^{(k)}_{i}(\mu) \rangle$ and $\langle c^{(k)}_{i}(\mu) \rangle = \langle {c'}^{(k)}_{i}(\mu) \rangle$, which is undesirable, since this configuration would state that $x_{i}$ complies perfectly with class $u$. In order to overcome this problem, a simple post-processing is necessary: For all components $u \in \mathcal{L}$ that do not share at least $1$ link with $x_{i}$, we deliberately set $\langle {t'}^{(j)}_{i}(\mu) \rangle$ and $\langle {c'}^{(j)}_{i}(\mu) \rangle$ to a high value. This high value must be greater than the largest variation that occurs in a component which shares a link with the data item under analysis. One may interpret this post-processing as a way to state that $x_{i}$ does not share any pattern formation with class $u$, since it is not even connected to it.

With all this information at hand, we are able to calculate $\Delta t^{(j)}_{i}(\mu)$ and $\Delta c^{(j)}_{i}(\mu),\forall j \in \mathcal{L}$, as follows:

\begin{align}
    \begin{split}
    \Delta t^{(j)}_{i}(\mu) = \frac{|\langle {t'}^{(j)}_{i}(\mu) \rangle - \langle {t}^{(j)}_{i}(\mu) \rangle|}{\sum_{u \in \mathcal{L}}{|\langle {t'}^{(u)}_{i}(\mu) \rangle - \langle {t}^{(u)}_{i}(\mu) \rangle|}}\\
    \Delta c^{(j)}_{i}(\mu) = \frac{|\langle {c'}^{(j)}_{i}(\mu) \rangle - \langle {c}^{(j)}_{i}(\mu) \rangle|}{\sum_{u \in \mathcal{L}}{|\langle {c'}^{(u)}_{i}(\mu) \rangle - \langle {c}^{(u)}_{i}(\mu) \rangle|}}  \end{split}
    \label{eq:deltaG-def-high-level}
\end{align}

\noindent where the denominator is introduced only for normalization matters. According to (\ref{eq:deltaG-def-high-level}), for insertions that result in a considerable variation of the component's transient and cycle lengths, $\Delta t^{(j)}_{i}(\mu)$ and $\Delta c^{(j)}_{i}(\mu)$ will be high. In view of (\ref{eq:f-function-def-high-level}), $T^{(j)}_{i}(\mu)$ and $C^{(j)}_{i}(\mu)$ are expected to be also high, yielding a low membership value predicted by the high level classifier $H^{(j)}_{i}$, as (\ref{eq:def-C-term-high-level}) reveals. On the other hand, for insertions that do not significantly interfere in the pattern formation of the data, $\Delta t^{(j)}_{i}(\mu)$ and $\Delta c^{(j)}_{i}(\mu)$ will be low, and, as a result, $T^{(j)}_{i}(\mu)$ and $C^{(j)}_{i}(\mu)$ are expected to be also low, producing a high membership value for the high level classifier $H^{(j)}_{i}$, as (\ref{eq:def-C-term-high-level}) exposes.

The network-based high level classifier quantifies the variations of the transient and cycle lengths of tourist walks with limited memory $\mu$ that occur in the class components when a test instance artificially joins each of them in isolation. According to (\ref{eq:def-C-term-high-level}), this procedure is performed for several values of the memory length $\mu$, ranging from $0$ (memoryless) to a critical value $\mu_{c}$. This is done in order to capture complex patterns of each of the representative class components in a local to global fashion. When $\mu$ is small, the walks tend to possess a small transient and cycle parts, so that the walker does not wander far away from the starting vertex. In this way, the walking mechanism is responsible for capturing the local structures of the class component. On the other hand, when $\mu$ increases, the walker is compelled to venture deep into the component, possibly very far away from its starting vertex. In this case, the walking process is responsible for capturing the global features of the component. In sum, the fundamental idea of the high level classifier is to make use of a mixture of local and global features of the class components by means of a combination of tourist walks with different values of $\mu$.

Finally, we intuitively explain the role of $p^{(j)} \in [0,1]$ in (\ref{eq:f-function-def-high-level}), i.e., the relative size of each component in the graph. In real-world databases, unbalanced classes are usually encountered. In general, a database frequently encompasses several classes of different sizes. A great portion of the network measures is very sensitive to the size of the components. In an attempt to soften this problem, we introduce in (\ref{eq:f-function-def-high-level}) the term $p^{(j)}$, which is the proportion of vertices that class $j$ has. Mathematically, it is given by:

\begin{align}
    p^{(j)} = \frac{1}{V} \sum_{u=1}^{V}{\mathds{1}_{\{y_{u} = j\}}},
    \label{eq:proportion-def}
\end{align}

\noindent where $V$ is the number of vertices and $\mathds{1}_{\{.\}}$ is the indicator function that yields $1$ if the argument is logically true, or $0$, otherwise. In view of the introduction of this mechanism, we expect to obviate the effects of unbalanced classes in the classification process.

\subsection{Algorithm}

For didactic purposes, Algorithm \ref{algorithm-high-level} enumerates the sequence of steps needed to perform a high level classification of a single test instance, according to the rules that we have described in this section.

\begin{algorithm}

 \caption{An overview of the high level classification procedure based on tourist walks for a single test instance.}
 \label{algorithm-high-level}
\end{algorithm}

\begin{enumerate}
    \item Construct a set of network components $\mathcal{G} = \{\mathcal{G}_{1}, \hdots, \mathcal{G}_{L}\}$, for each class
 from the vector-based training set by using the combined $k$-NN and $\epsilon$ rules from the training phase;

    \item Calculate the average transient and cycle lengths for each network's component, representing a class of training data;

    \item Insert a test instance $x_i$ into the formed graph by using the $k$-NN and $\epsilon$ rules from the classification phase;

    \item Calculate the new average transient and cycle lengths for each network's component that received at least one link from the test instance;

    \item Calculate the transient and cycle length variations by using Eq. (\ref{eq:f-function-def-high-level});

    \item Produce the high level decision value by using Eq. (\ref{eq:def-C-term-high-level}).
\end{enumerate}

\subsection{Linking the High Level Classifier based on Tourist Walks and the General Framework introduced in \cite{Silva2012HighLevel}}

One may wonder how the high level classifier based on tourist walks given by (\ref{eq:def-C-term-high-level}) is plugged into the general framework for high level classification introduced in \cite{Silva2012HighLevel}. The following proposition links both approaches.

\begin{proposition}

    The high level classifier based on tourist walks, whose decision equations are:

    \begin{align}
    H^{(j)}_{i} = \frac{\sum_{\mu=0}^{\mu_{c}}{\left[\alpha_{t}(1 - T_{i}^{(j)}(\mu)) + \alpha_{c}(1 - C_{i}^{(j)}(\mu)) \right]} }{\sum_{g\in \mathcal{L}}{\sum_{\mu=0}^{\mu_{c}}{\left[\alpha_{t}(1 - T_{i}^{(g)}(\mu)) + \alpha_{c}(1 - C_{i}^{(g)}(\mu)) \right]}}},
    \label{eq:def-C-term-high-level-prop}
    \end{align}
    \begin{align}
    \begin{split}
    T_{i}^{(j)}(\mu) = \Delta t^{(j)}_{i}(\mu) p^{(j)}\\
    C_{i}^{(j)}(\mu) = \Delta c^{(j)}_{i}(\mu) p^{(j)},
    \end{split}
    \label{eq:f-function-def-high-level-prop}
    \end{align}

    \noindent is a particular form of the generic high level framework given in \cite[Eqs. (5) and (7)]{Silva2012HighLevel}:

    \begin{align}
        H^{(j)}_{i} = \frac{\sum_{u = 1}^{m}{\alpha(u)\left[1 - f^{(j)}_{i}(u) \right]}}{\sum_{g \in \mathcal{L}}{\sum_{u = 1}^{m}{\alpha(u)\left[1 - f^{(g)}_{i}(u) \right]}}},
        \label{eq:def-C-term-prop}
    \end{align}
    \begin{align}
        f^{(j)}_{i}(u) = \Delta G^{(j)}_{i}(u) p^{(j)}.
        \label{eq:f-function-def-prop}
    \end{align}

\end{proposition}

\begin{proof}

    First, one can see that (\ref{eq:def-C-term-high-level-prop}) can be written as:

    \begin{align}
    H^{(j)}_{i} = \frac{\sum_{\mu=1}^{2\mu_{c} + 2}{v(\mu)(1 - V_{i}^{(j)}(\mu))}}{\sum_{g\in \mathcal{L}}{\sum_{\mu=1}^{2\mu_{c} + 2}{v(\mu)(1 - V_{i}^{(g)}(\mu))}}},
    \label{eq:midway-point-prop}
    \end{align}

    \noindent where:

    \begin{align}
         v(\mu) = \left\{
                  \begin{array}{l l}
                    \alpha_{t} & \quad \text{if $\mu \le \mu_{c} + 1$}\\
                    \alpha_{c} & \quad \text{if $\mu > \mu_{c} + 1$}\\
                  \end{array} \right.,
    \end{align}

    \noindent and

    \begin{align}
         V_{i}^{(j)}(\mu) = \left\{
                  \begin{array}{l l}
                    T_{i}^{(j)}(\mu) & \quad \text{if $\mu \le \mu_{c} + 1$}\\
                    C_{i}^{(j)}(\mu) & \quad \text{if $\mu > \mu_{c} + 1$}\\
                  \end{array} \right.,
    \end{align}

    \noindent are piecewise functions defined over the domain $\{1, \hdots, 2\mu_{c} + 2\}$. In view of \cite[Eq. (6)]{Silva2012HighLevel}, one must have that:

     \begin{align}
         \sum_{\mu=1}^{2\mu_{c} + 2}{v(\mu)} = 1 \Rightarrow (\mu_{c} + 1)(\alpha_{t} + \alpha_{c}) = 1.
         \label{eq:midway-2}
     \end{align}

     Moreover, by equivalence, comparing the upper limits of the summations of (\ref{eq:def-C-term-prop}) and (\ref{eq:midway-point-prop}), one has that:

     \begin{align}
         m =  2\mu_{c} + 2.
     \end{align}

     Using \cite[Eq. (6)]{Silva2012HighLevel} on (\ref{eq:midway-2}), one has:

     \begin{align}
         (\mu_{c} + 1)(\alpha_{t} + \alpha_{c}) = \sum_{\mu=1}^{2\mu_{c} + 2}{v(\mu)} = 1 = \sum_{u=1}^{m}{\alpha(u)} = \sum_{u=1}^{2\mu_{c} + 2}{\alpha(u)}.
     \end{align}

     Therefore, by polynomial equivalence, one can take:

     \begin{align}
        \alpha(u) = v(u) = \left\{
                  \begin{array}{l l}
                    \alpha_{t} & \quad \text{if $u \le \mu_{c} + 1$}\\
                    \alpha_{c} & \quad \text{if $u > \mu_{c} + 1$}\\
                  \end{array} \right..
     \end{align}

     Using the same reasoning, we have that:

     \begin{align}
        f^{(j)}_{i}(u) = V_{i}^{(j)}(u) = \left\{
                  \begin{array}{l l}
                    T_{i}^{(j)}(u) & \quad \text{if $u \le \mu_{c} + 1$}\\
                    C_{i}^{(j)}(u) & \quad \text{if $u > \mu_{c} + 1$}\\
        \end{array} \right..
     \end{align}

     Finally, by looking at (\ref{eq:f-function-def-high-level-prop}) and (\ref{eq:f-function-def-prop}), we can infer that the tourist walk based classifier can be coupled into the hybrid framework by taking $2\mu_{c} + 2$ network measures, such that the first $\mu_{c} + 1$ are the transient lengths with increasing memory lengths, each of which weighted by $\alpha_{t}$, and the remaining $\mu_{c} + 1$ are cycle lengths with increasing memory lengths, each of which weighted by $\alpha_{c}$. In view of this, the high level classifier based on tourist walks is, in fact, a particular implementation of the generic high level classifier.
\end{proof}

\section{Computer Simulations}
\label{Computer-Simulations}

In this section, we present computer simulation results in order to assess the effectiveness of the proposed hybrid classification model based on tourist walks. The error estimation method in these simulations is set to be the stratified $10$-fold cross-validation.

\subsection{Motivating and Illustrative Examples}
\label{sec:how-it-works}

In this section, we provide simple examples with the goal of showing the mechanics of the proposed hybrid classification technique.  For this end, we simplify the parameters selection procedure, as follows: the weights given for the transient and cycle lengths are the same, i.e., $\alpha_{t} = \alpha_{c} = 0.5$; and the critical tourist walk length $\mu_{c}$ is set as being the size of the smallest class in the problem. Here, we design particular situations in which the use of the high order of learning is welcomed, since the reliance on mere physical measures would probably deceive a low level classifier.

As an introductory example, consider the synthetic data set supplied in Fig. \ref{fig:Motivation-2-Classes} in the Introduction section. Our goal is to classify the two triangle-shaped data items ($\mathbf{X}_{\mathrm{test}}$). The items in $\mathbf{X}_{\mathrm{test}}$ are inserted one by one using $\epsilon=0.05$ and $k = 2$ to construct the network components before and after each insertion.  For the purposes of this simulation, once a test instance is classified, it is discarded. With respect to the low level classifier, a fuzzy SVM classifier is utilized \cite{Lin2002}, which is equipped with optimization method criterion defined as the Karush-Kuhn-Tucker violation fixed at $10^{-3}$ (the same condition suggested by \cite{Hsu2002}). The following four kernels are used: (i) Linear: $u \cdot v$, (ii) RBF: $\mathrm{exp}(-\gamma \parallel u - v \parallel^{2})$, (iii) Sigmoid: $\mathrm{tanh}(\gamma u \cdot v + c)$, and (iv) Polynomial: $(c + u \cdot v)^{d}$. Additionally, the cost parameter $C$ must be set in all simulations. For each low level classifier, we apply a fine-tuning parameter phase in order to choose a model with optimized classification results.

The networked data set exhibited in Fig. \ref{fig:Motivation-2-Classes} has $2$ classes, a red or circular-shaped class ($16$ vertices) and a blue or square-shaped class ($58$ vertices). Consequently, the proportions of vertices pertaining to each class are: $p^{(\mathrm{red})} = 21.64\%$ and $p^{(\mathrm{blue})} = 78.36\%$. Clearly, while the circular-shaped (red) class carries a perceivable pattern, which, geometrically speaking, is a lozenge-shaped lattice disposed in a $2$-dimensional space, the blue (``square") class does not indicate a strong pattern as the first class. In this particular situation, we will see that the results produced by the SVM by itself are not satisfactory. Even though the kernel is responsible for spawning the data items into a higher space via a nonlinear transformation, the test instances in question may only be distinguished by using semantic knowledge, i.e., the topological features of the classes. It is worth stressing that special kernels could be created in order to solve specific data distributions, but they would be totally peculiar for the problems at hand, making their usability constrained to a very few real situations. On the other side, the proposed high level technique captures pattern formation of the data in a general manner.

\begin{table*}[htbp]
  \centering
  \footnotesize
  \caption{Results obtained for the classification of the triangle-shaped test instances exhibited in Fig. \ref{fig:Motivation-2-Classes}.}
    \begin{tabular}{rrcrcccccccc}
    \addlinespace
    \toprule
           & \multicolumn{1}{c}{\textbf{Kernel}} & \multicolumn{4}{c}{\textbf{Optimized Parameters}} & \multicolumn{2}{c}{\textbf{$\lambda = 0$}} & \multicolumn{2}{c}{\textbf{$\lambda = 0.5$}} & \multicolumn{2}{c}{\textbf{$\lambda = 0.9$}} \\
           &        & \textbf{C} & \textbf{$\gamma$} & \textbf{c} & \textbf{d} & \textbf{Red Class} & \textbf{Blue Class} & \textbf{Red Class} & \textbf{Blue Class} & \textbf{Red Class} & \textbf{Blue Class} \\
    \midrule
    \multicolumn{1}{c}{\multirow{4}[0]{*}{\begin{sideways}\textbf{Left-most}\end{sideways}}} & \textbf{Linear} & $2^{1}$      & \multicolumn{1}{c}{N/A} & N/A    & N/A    & 0.40   & \textbf{0.60} & \textbf{0.54} & 0.46   & \textbf{0.64} & 0.36 \\
    \multicolumn{1}{c}{} & \textbf{RBF} & $2^{0}$      & \multicolumn{1}{c}{$2^{-6}$} & N/A    & N/A    & 0.47 & \textbf{0.53}   & \textbf{0.60} & 0.40   & \textbf{0.67} & 0.33 \\
    \multicolumn{1}{c}{} & \textbf{Sigmoid} & $2^{0}$      & \multicolumn{1}{c}{$2^{-1}$} & $-4.5$  & N/A    & 0.46 & \textbf{0.54}   & \textbf{0.59} & 0.41   & \textbf{0.67} & 0.33 \\
    \multicolumn{1}{c}{} & \textbf{Polynomial} & $2^{1}$      & \multicolumn{1}{c}{$2^{-4}$} & $3.0$  & $4$ & 0.44 & \textbf{0.56}   & \textbf{0.57} & 0.44   & \textbf{0.65} & 0.35 \\
    \midrule
    \multicolumn{1}{c}{\multirow{4}[0]{*}{\begin{sideways}\textbf{Right-most}\end{sideways}}} & \textbf{Linear} & $2^{-1}$      & \multicolumn{1}{c}{N/A} & N/A    & N/A    & 0.24 & \textbf{0.76}   & 0.40 & \textbf{0.60}   & \textbf{0.51} & 0.49 \\
    \multicolumn{1}{c}{} & \textbf{RBF} & $2^{-1}$      & \multicolumn{1}{c}{$2^{-5}$} & N/A    & N/A    & 0.29 & \textbf{0.71}   & 0.45 & \textbf{0.55}   & \textbf{0.57} & 0.43 \\
    \multicolumn{1}{c}{} & \textbf{Sigmoid} & $2^{-1}$      & \multicolumn{1}{c}{$2^{-3}$} & $-6.5$  & N/A    & 0.31 & \textbf{0.69}   & 0.47 & \textbf{0.53}   & \textbf{0.58} & 0.42 \\
    \multicolumn{1}{c}{} & \textbf{Polynomial} & $2^{-1}$      & \multicolumn{1}{c}{$2^{-5}$} & $5.0$  & $2$ & 0.28 & \textbf{0.72}   & 0.44 & \textbf{0.56}   & \textbf{0.57} & 0.43 \\
    \bottomrule
    \end{tabular}%
  \label{tab:grid-lattice}%
\end{table*}%

Table \ref{tab:grid-lattice} reports the classification results of the two existing test instances for all the $4$ kernels. For each kernel, the final prediction is given for three different values of the compliance term $\lambda$: $\lambda = 0$ (pure SVM), $\lambda = 0.5$ (equal weight for the decision of the SVM and the high level classifier), and $\lambda = 0.9$ (a high weight for the high level classifier and a low weight for the SVM). We can clearly see that a pure SVM with well-known kernels is not able to correctly classify the data items shown in Fig. \ref{fig:Motivation-2-Classes}. However, if we make our final decision based on a mixture of the SVM and the high level classifier, one can verify that correct results can be obtained, i.e., the two triangle-shaped data items are classified to red or circular-shaped class.

\begin{figure} [htb]
    \centering
    \includegraphics[scale = 0.25]{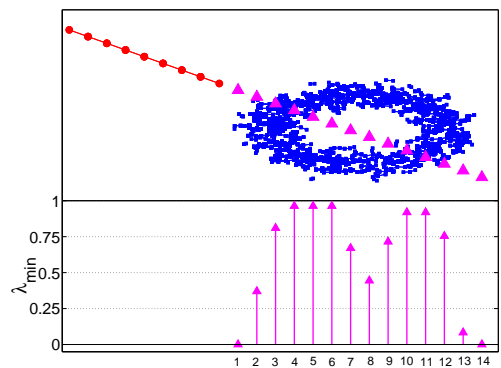}
    \caption{Minimum value of the compliance term, $\lambda_{\mathrm{\mathrm{min}}}$, that results in the correct classification
    of the missing test instances. Traditional techniques would definitely fail in correctly classifying the straight line that diametrally     crosses the densely connected component pertaining to the blue or ``square" class.}%
    \label{fig:Devious-Straight-Line-LambdaAnalysis}
\end{figure}

Now, consider the classification problem arranged in Fig. \ref{fig:Devious-Straight-Line-LambdaAnalysis}. Here, we are going to empirically calculate the minimum required compliance term $\lambda_{\mathrm{min}}$ for which the data items from the test set are classified as members of the red or circular-shaped class. In the figure, one can see that there is a segment of line representing the red or circular-shaped class ($9$ vertices) and also a condensed rectangular class outlined by the blue or square-shaped class ($1000$ vertices). The network formation in the training and test phases uses $k=1$ and $\epsilon = 0.07$ (this radius covers, for any vertex in the straight line, $2$ adjacent vertices, except for the vertices in each end). The fuzzy SVM technique with RBF kernel ($C = 2^{2}$ and $\gamma = 2^{-1}$) is employed as the traditional low level classifier.
The task is to classify the $14$ test instances depicted by the big triangle-shape items from left to right. After a test instance is classified, it is incorporated to the training set with the corresponding predicted label.  The graphic embedded in Fig. \ref{fig:Devious-Straight-Line-LambdaAnalysis} shows the minimum required value of $\lambda_{\mathrm{min}}$ for which the the triangle-shaped items are classified as members of the red or circular-shaped class. This graphic is constructed according to the triangle-shaped element that is exactly at the same position with respect to the $x$-axis in the scatter plot drawn above. For example, the first triangle-shaped data item can be correctly classified if one chooses $\lambda \ge \lambda_{\mathrm{min}} \approx 0.37$. The second and third data items would require at least $\lambda_{\mathrm{min}} = 0.81$ and $\lambda_{\mathrm{min}} = 0.96$, respectively, and so on. Specifically, as the straight line crosses the condensed region of the blue or square-shaped class, the compliance term approaches $\lambda \rightarrow 1$, since one cannot establish its decision based on the low level classifier, because it would erroneously decide favorable to the blue or square-shaped class.

\subsection{Parameter Sensitivity Analysis}
\label{sec:parameter-sensitivity}

In this section, we will perform several simulations with the goal of better understanding the
influence of each parameter in the model.

\subsubsection{Influence of the Compliance Term for Different High Level Classifiers}

In this section, we study the influence of the compliance term $\lambda$ for three different types of high level classifiers, as follows: (i) one constructed solely using the component's cycle length; (ii) one built exclusively using the component's transient length; and (iii) the best weighted combination of these two measures. The optimization process is conducted by finding $\alpha_{t} \times \alpha_{c} \in \{0, 0.1,\hdots, 1\} \times \{0, 0.1,\hdots, 1\}$ (search space), subjected to $\alpha_{t} + \alpha_{c} = 1$, which results in the highest accuracy rate of the model. Here, we show that, in general, the transient and cycles lengths are not correlated. The objective of this section is to make this point clear and demonstrate that, when they act together, they are able to produce better accuracy rates. In order to conduct these tests, classes with Gaussian distributions are used. The network in the classification phase is constructed with $k=1$ and  $\epsilon=0.05$. The same $\epsilon$ is employed in the classification phase. The similarity measure is simply given by the reciprocal of the Euclidean measure. For the traditional classifier, the fuzzy SVM with RBF kernel is employed with $C=2^{4}$ and $\gamma = 2^{-2}$. Finally, this process is repeated $100$ times and the mean and the corresponding standard deviation for each value of $\lambda$ is reported.

Now, we advance to our first experiment in which we consider a rather simple classification scenario, where the classes are completely separated from each other, as Fig. \ref{fig:Gauss-Separated} exhibits. Figure \ref{fig:Gauss-Separated-ClassificationRate} depicts the accuracy rate of the proposed classifier vs. $\lambda$ for the three types of high level classifiers discussed before. The best weighted combination of the cycle and transient lengths is $\alpha_{c} = 0.6$ and $\alpha_{t} = 0.4$. One can see that the high level classifier constructed by the combination of the transient and cycle lengths is able to outperform the other high level classifiers constructed using only one of these measures. In addition, when $\lambda = 0$ (only the usage of the traditional classifier), almost no wrong labels are assigned. On the other hand, as $\lambda$ increases, the accuracy rate of the proposed technique starts to monotonically decrease. This is predictable, since the two classes are similar; hence, the network measures associated to each representative component are almost equivalent. Arguments such that there is no mixture between the classes and the spatial correlations of both classes are similar reinforce this phenomenon. Therefore, it is natural that the high level classifier will become confused in classifying these data items under such conditions. This example shows that only the high level classifier is insufficient to get good classification results.

\begin{figure} [!Htb]
    \centering
    \subfloat[]
    {\includegraphics[scale = 0.16]{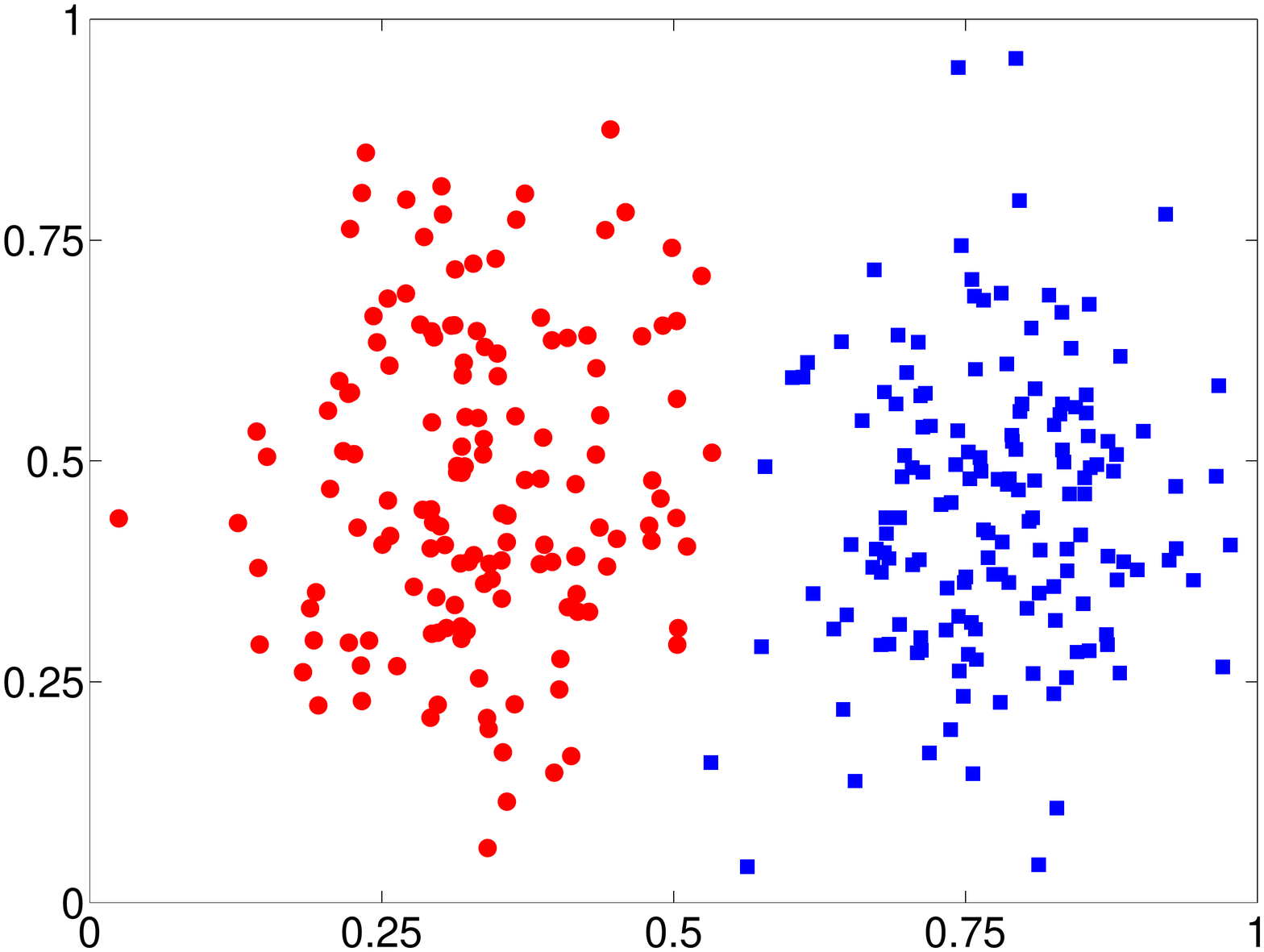}\label{fig:Gauss-Separated}}%
    \subfloat[]
    {\includegraphics[scale = 0.16]{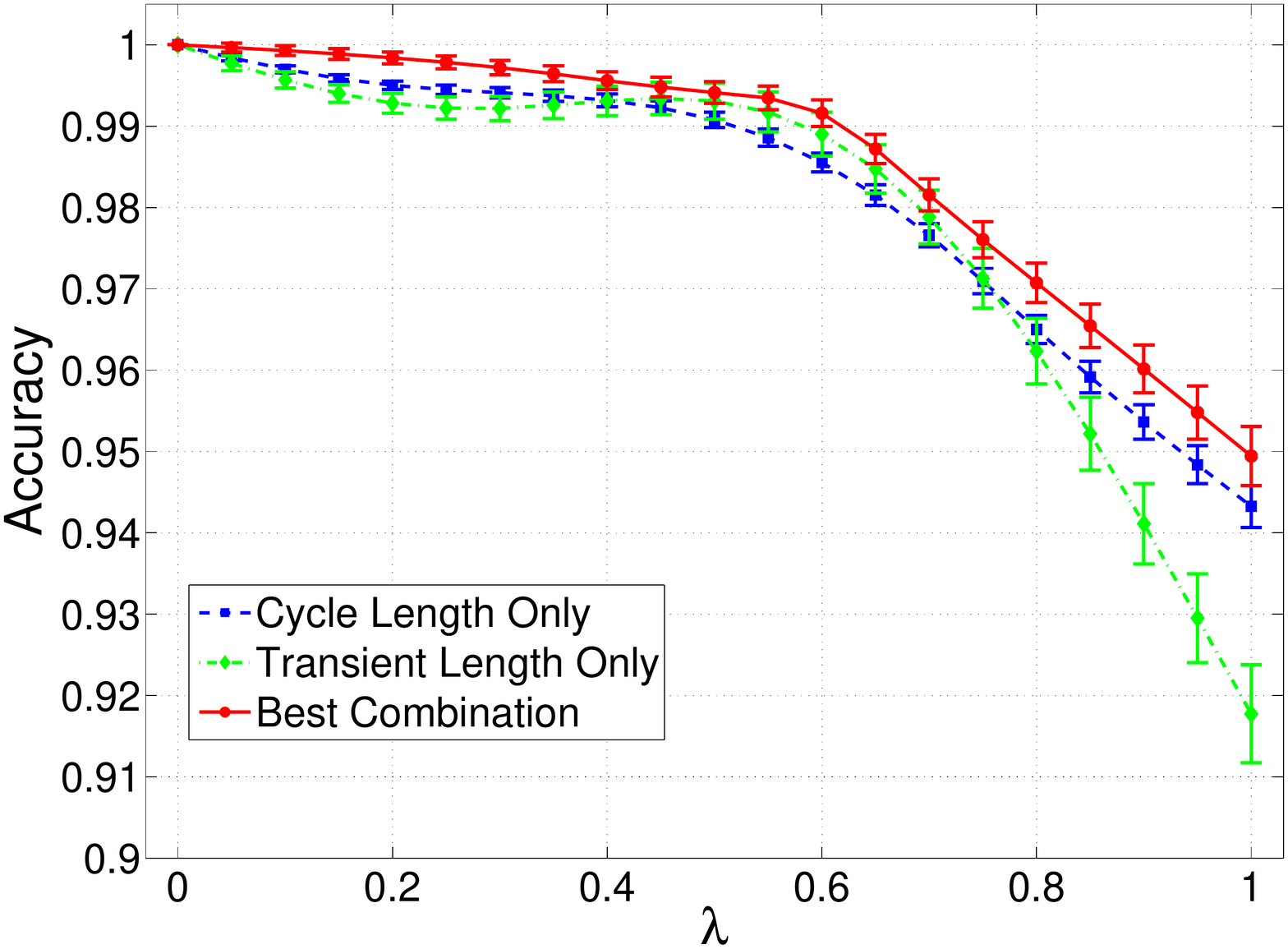}\label{fig:Gauss-Separated-ClassificationRate}}\\%

    \subfloat[]
    {\includegraphics[scale = 0.16]{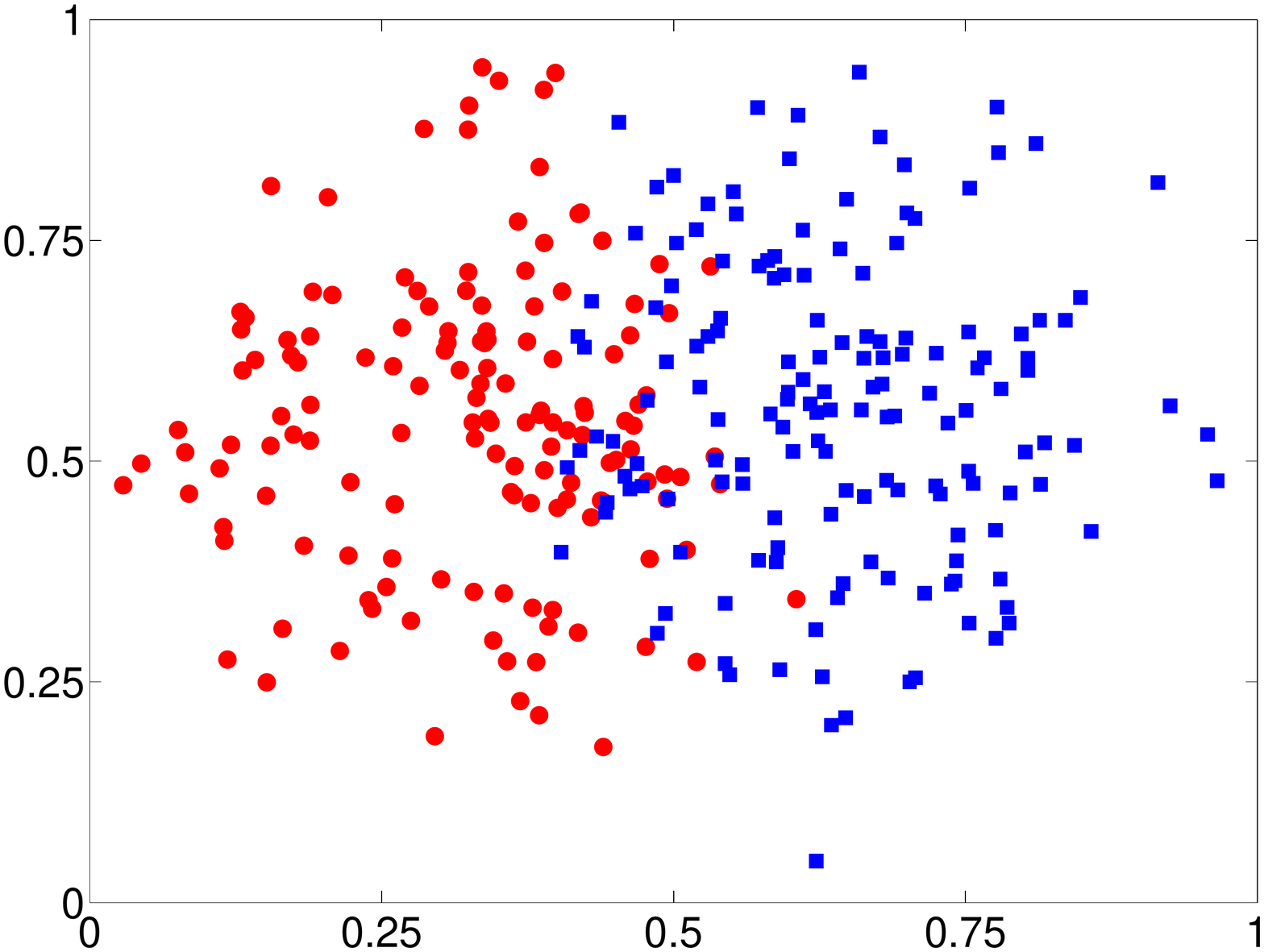}\label{fig:Gauss-Slightly-Mixed}}%
    \subfloat[]
    {\includegraphics[scale = 0.16]{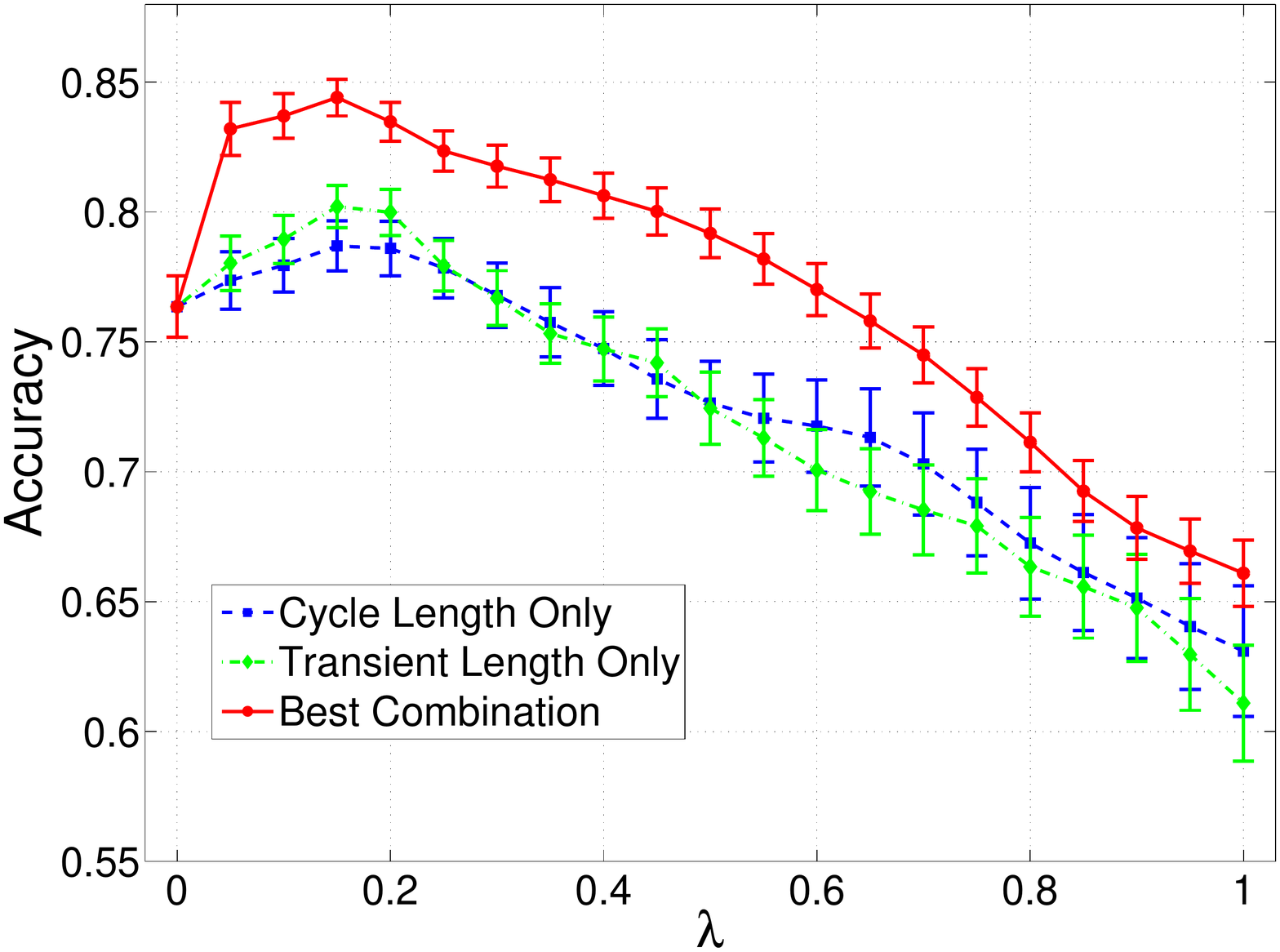}\label{fig:Gauss-Slightly-Mixed-ClassificationRate}}%

    \subfloat[]
    {\includegraphics[scale = 0.16]{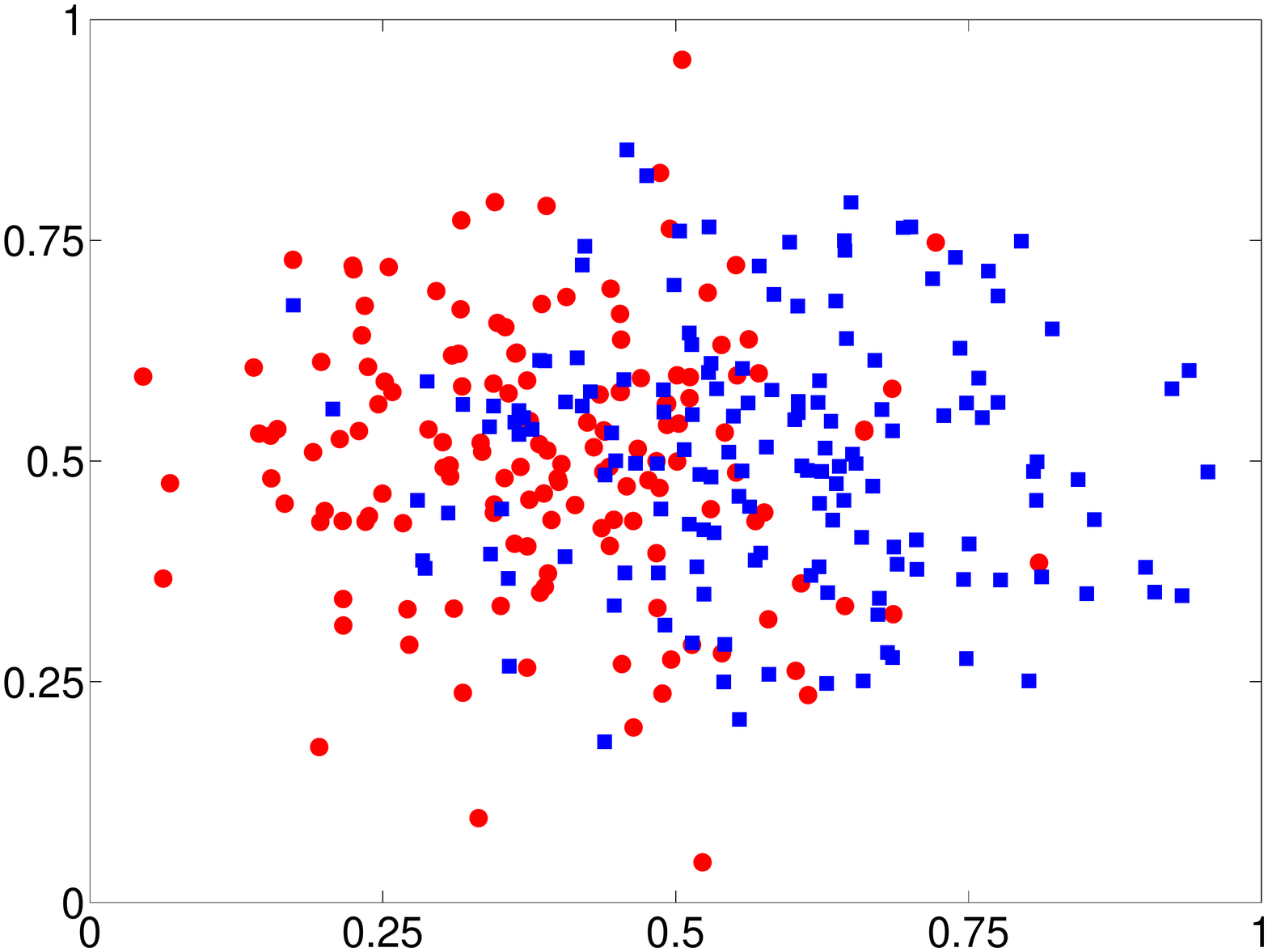}\label{fig:Gauss-Heavily-Mixed}}%
    \subfloat[]
    {\includegraphics[scale = 0.16]{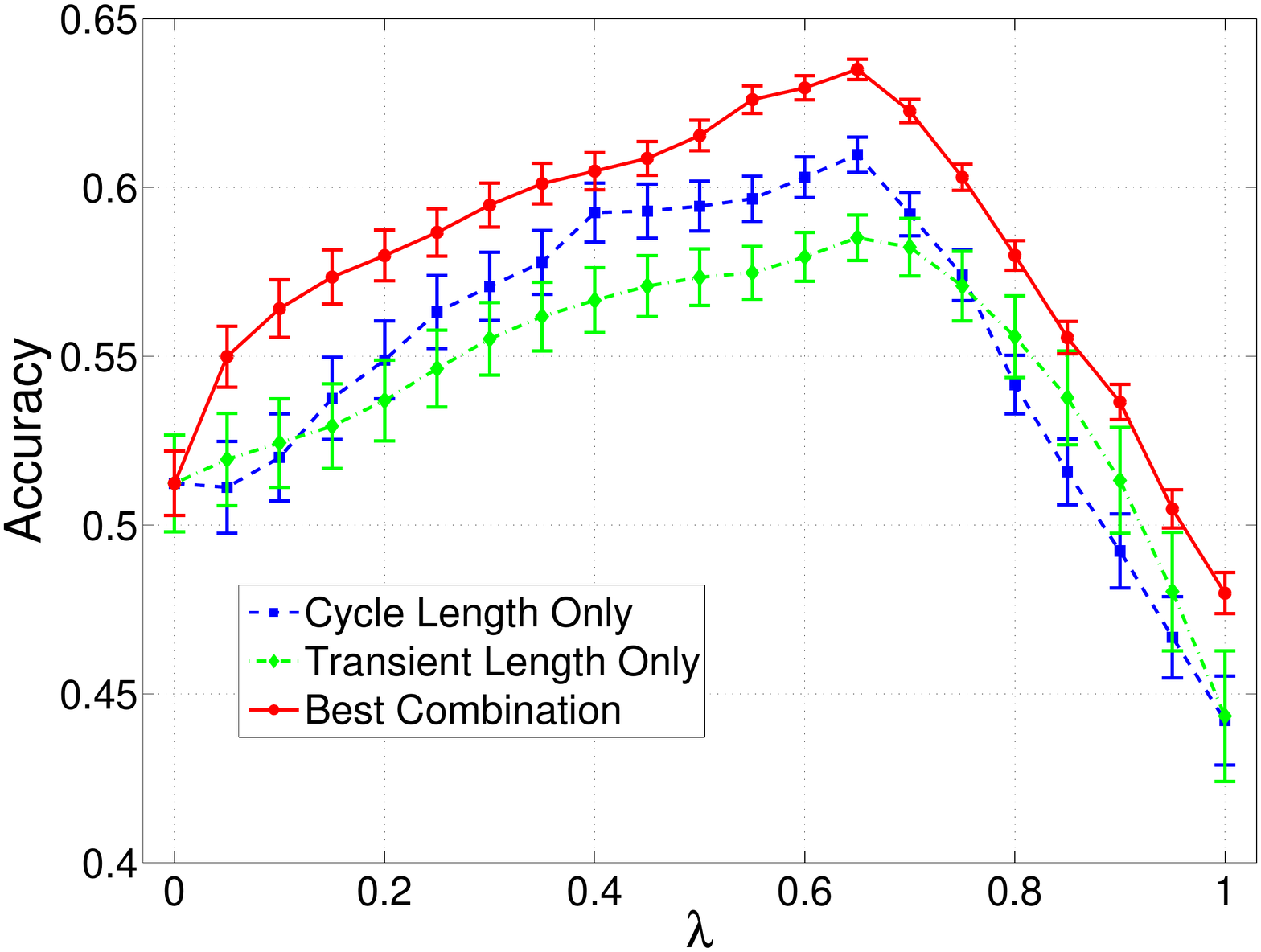}\label{fig:Gauss-Heavily-Mixed-ClassificationRate}}%

    \caption{Analysis of accuracy rate vs. the compliance term with three distinct high level classifiers. (a) Two integrally separated classes. (b) Impact of $\lambda$ on the network of Fig. \ref{fig:Artificial-data}a. (c) Two classes slightly mixed. (d) Impact of $\lambda$ on the network of Fig. \ref{fig:Artificial-data}c. (e) Two classes heavily mixed. (f) Impact of $\lambda$ on the network of Fig. \ref{fig:Artificial-data}e.
    }%
    \label{fig:Artificial-data}
\end{figure}

Next, let us examine the problem of classification illustrated in Fig. \ref{fig:Gauss-Slightly-Mixed}, where the two classes slightly collide with each other. As a result of this phenomenon, a conflicting region is constituted. The data items which reside in this region are most likely to be misclassified by a pure traditional classifier. Similarly to the previous case, Fig. \ref{fig:Gauss-Slightly-Mixed-ClassificationRate} which displays the accuracy rate of the model. In this case, the best weighted combination of transient and cycle lengths is $\alpha_{t} = 0.7$ and $\alpha_{c} = 0.3$. One can see that a mixture of low level and high level classifiers does supply a boost in the accuracy rate of the model. Specifically, the highest accuracy rate occurs when $\lambda \in \{0.15, 0.2\}$. The region of coalescence of both classes influences in the network construction, in such a way that the representative components of each class become slightly different. These arguments explain that a little relevance to the high level classifier decision can cause the accuracy rate to increase, since the network measures describing each component are slightly distinct.

We now proceed to inspect the classification problem proposed in Fig. \ref{fig:Gauss-Heavily-Mixed}. In this particular situation, the two classes heavily collide with each other. Due to that, the two classes become almost indistinguishable. In consequence of the network formation technique, two very distinct and representative components will arise from these two classes. In this special scenario, the two components that comprise the network are expected to be possess different network properties, i.e., some unique structural pattern for each class is hoped for emerging.  Figure \ref{fig:Gauss-Heavily-Mixed-ClassificationRate} depicts the accuracy rate reached by the classifier against various values of $\lambda$. The best weighted combination here is $\alpha_{t} = 0.2$ and $\alpha_{c} = 0.8$. One can see that the accuracy rate keeps on increasing until a high value of the compliance term, namely $\lambda = 0.65$, where it achieves $65\%$, against $54\%$ with $\lambda = 0$ for the high level classifier combining the transition and cycle lengths. This phenomenon is precisely explained by the complex structural patterns that each component exhibits, in which merely traditional classifiers would not be able to capture; therefore, a heavy weight on the high level classifier decision was responsible for this significant increase in the accuracy rate.

\subsubsection{Influence of the Critical Memory Length}

As we have seen, the high level classifier makes its prediction by using combinations of several tourist walks with memory lengths $\{0,1,\hdots,\mu_{c}\}$. A natural question that arises is: is it really necessary to perform the computation of tourist walks with $\mu$ ranging from $0$ to a maximum feasible number, i.e., the number of vertices in a component? In this section, we will empirically show that, usually, it is not necessary to perform all these computations. For this end, we will reinforce this argument by verifying this behavior in a synthetic data set and two real-world data sets.

We start out by revisiting the synthetic data set given in Fig. \ref{fig:Motivation-2-Classes} shown in the Introduction section. The behavior of the transient and cycle lengths are displayed in Figs. \ref{fig:Artificial-transient-length} and \ref{fig:Artificial-cycle-length}, respectively. One can note that these two dynamical measures vary as $\mu$ increases up to a point where they reach a steady region in which no more oscillating is verified. When this happens, we say that the dynamics of the tourist walk have reached a \emph{saturation point}, in the sense that further computations of tourist walks with larger memory lengths are redundant. Moreover, one can see that this saturation point is reached very quickly in relation to the graph components.

\begin{figure} [!htb]
    \centering
    \subfloat[]
    {\includegraphics[scale = 0.16]{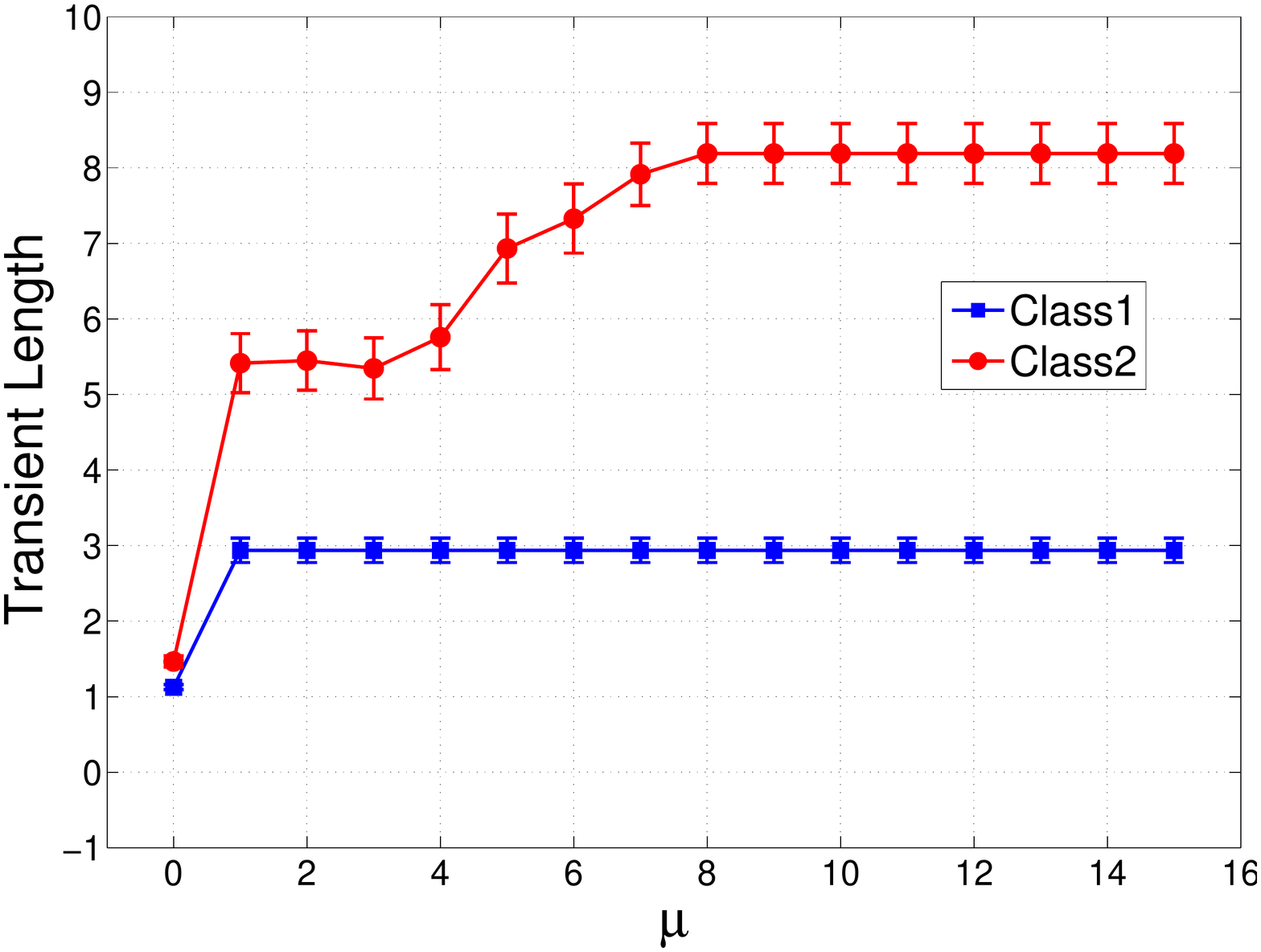}\label{fig:Artificial-transient-length}}%
    \subfloat[]
    {\includegraphics[scale = 0.16]{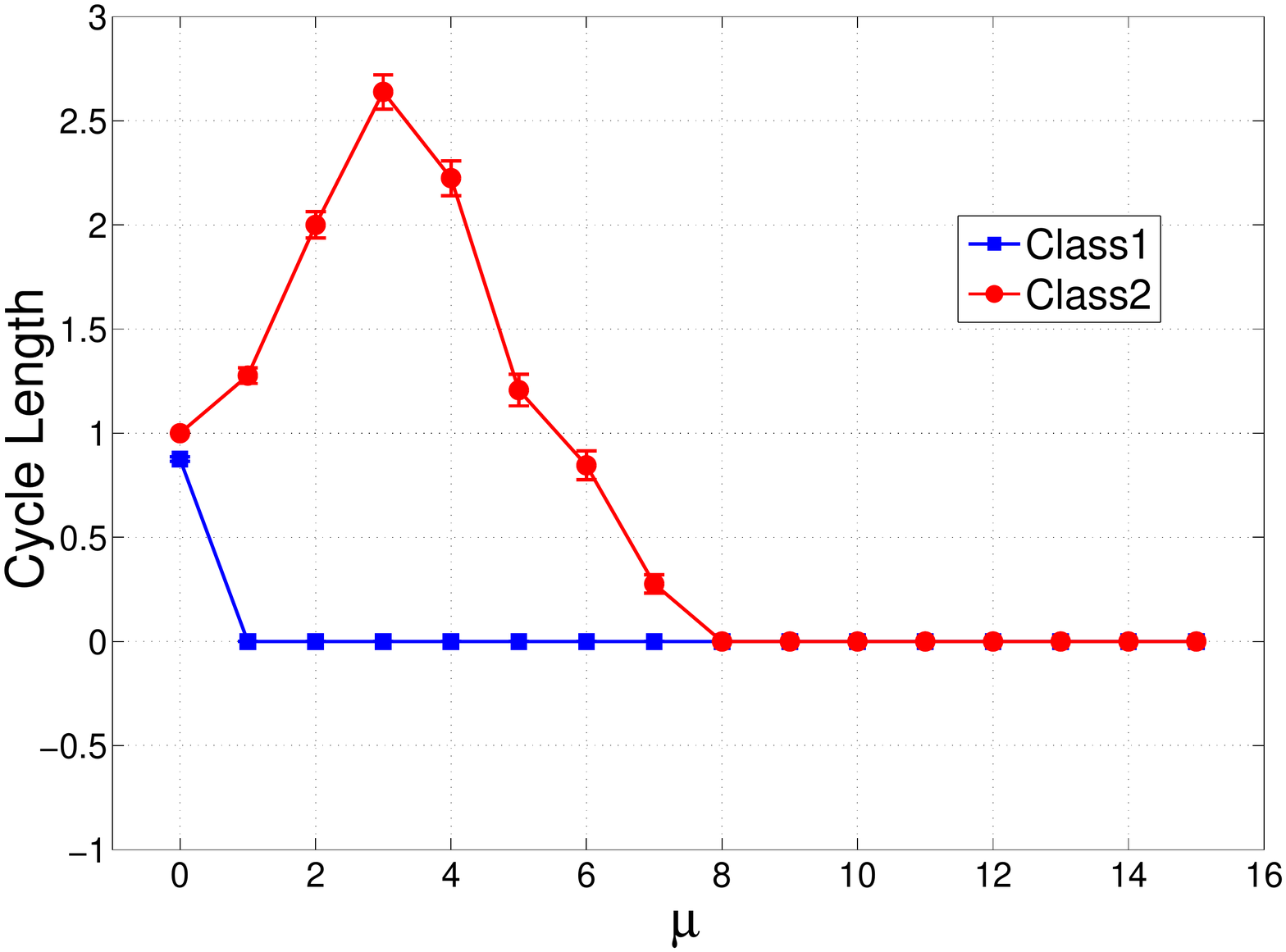}\label{fig:Artificial-cycle-length}}%

    \caption{Behavior of the transient and cycle lengths of the synthetic data set displayed in Fig. \ref{fig:Motivation-2-Classes}.}%
    \label{fig:artificial-transient-cycle}
\end{figure}

\begin{figure} [!htb]
    \centering
    \subfloat[]
    {\includegraphics[scale = 0.16]{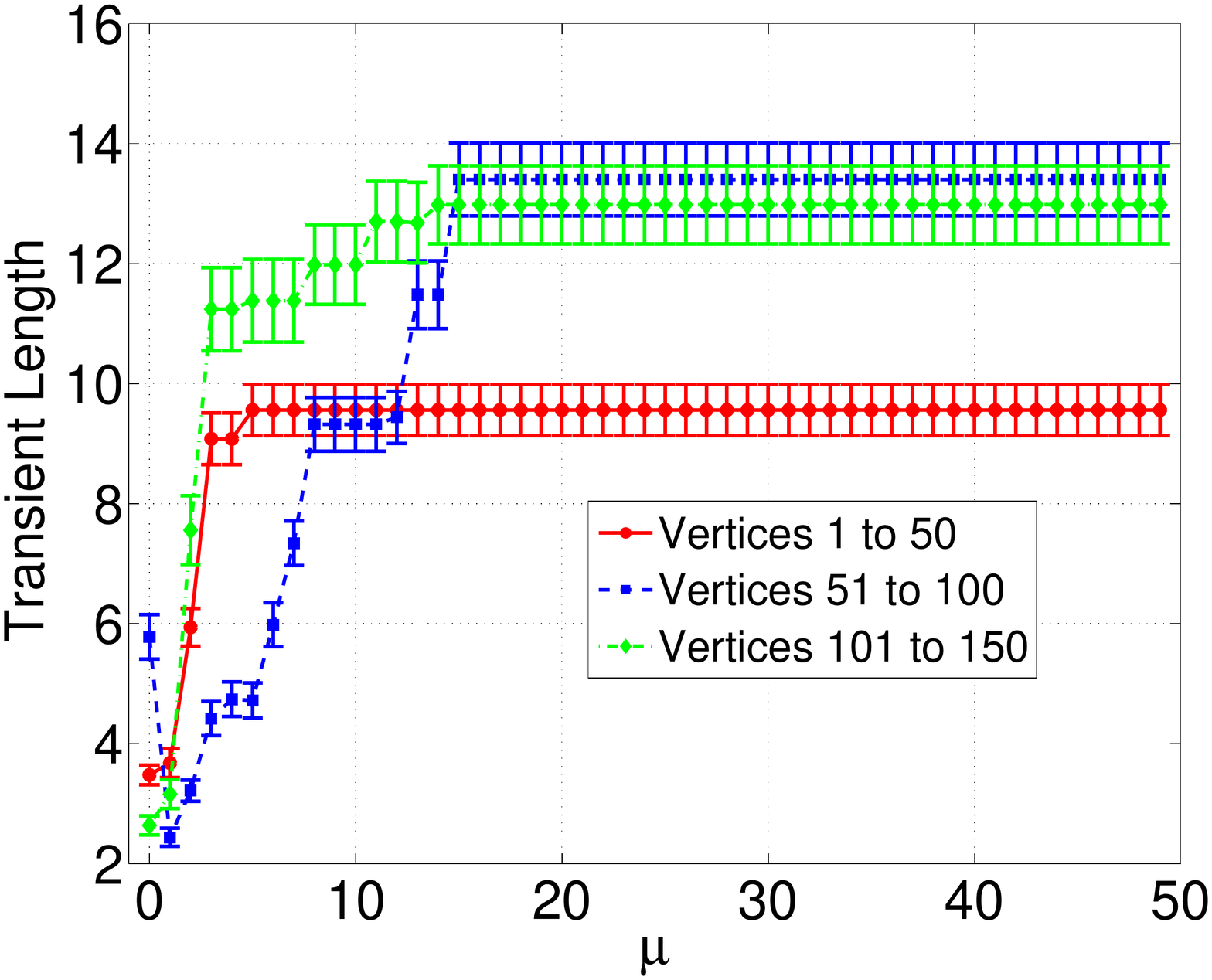}\label{fig:Iris-transient-length}}%
    \subfloat[]
    {\includegraphics[scale = 0.16]{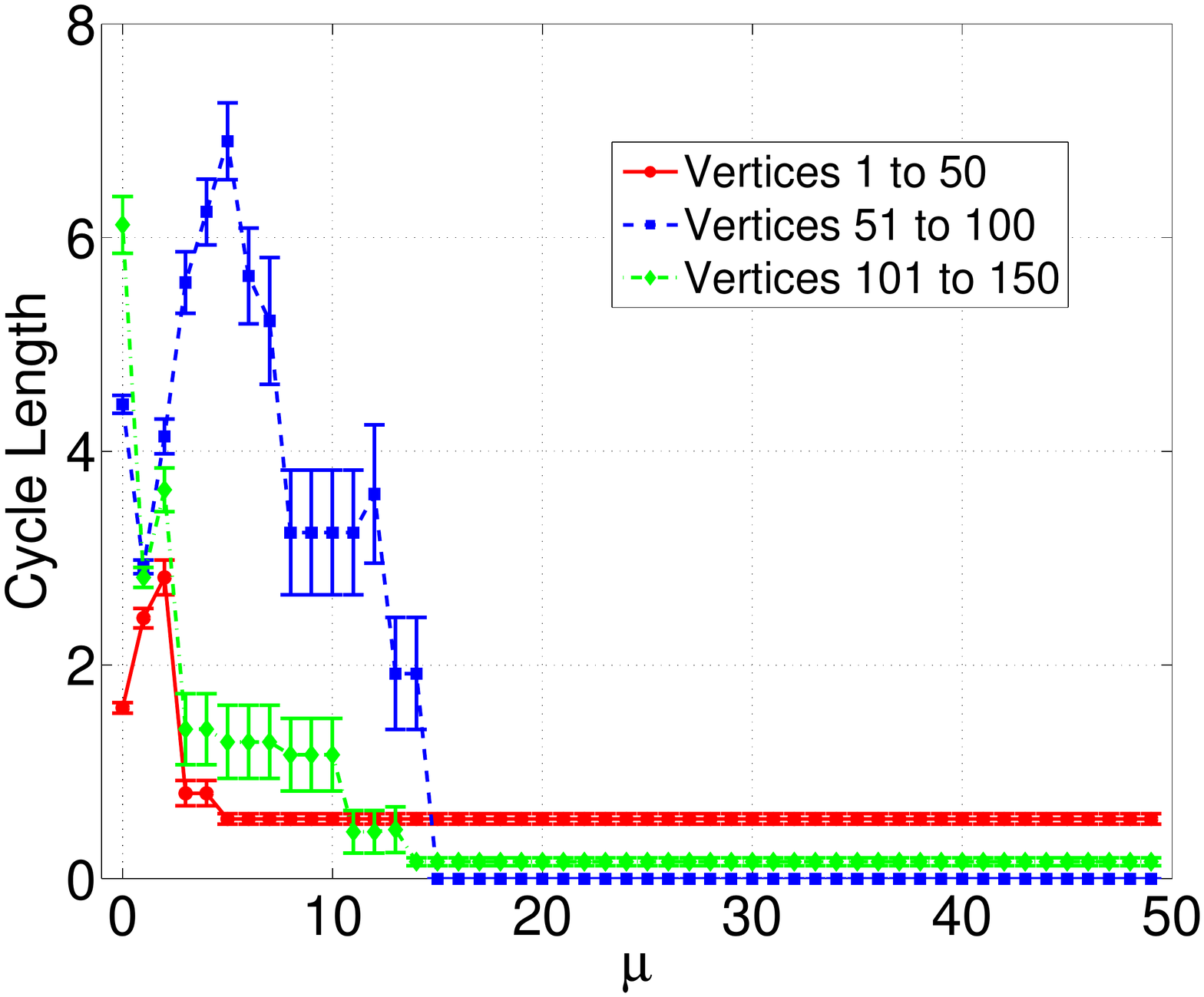}\label{fig:Iris-cycle-length}}%

    \caption{Behavior of the transient and cycle lengths of the Iris data set. Network's parameters: $k = 1$ and $\epsilon = 0.03$. (a) Transient lengths vs. $\mu_{c}$. (b) Cycle lengths vs. $\mu_{c}$.
    }%
    \label{fig:iris-transient-cycle}
\end{figure}

\begin{figure} [!htb]
    \centering
    \subfloat[]
    {\includegraphics[scale = 0.16]{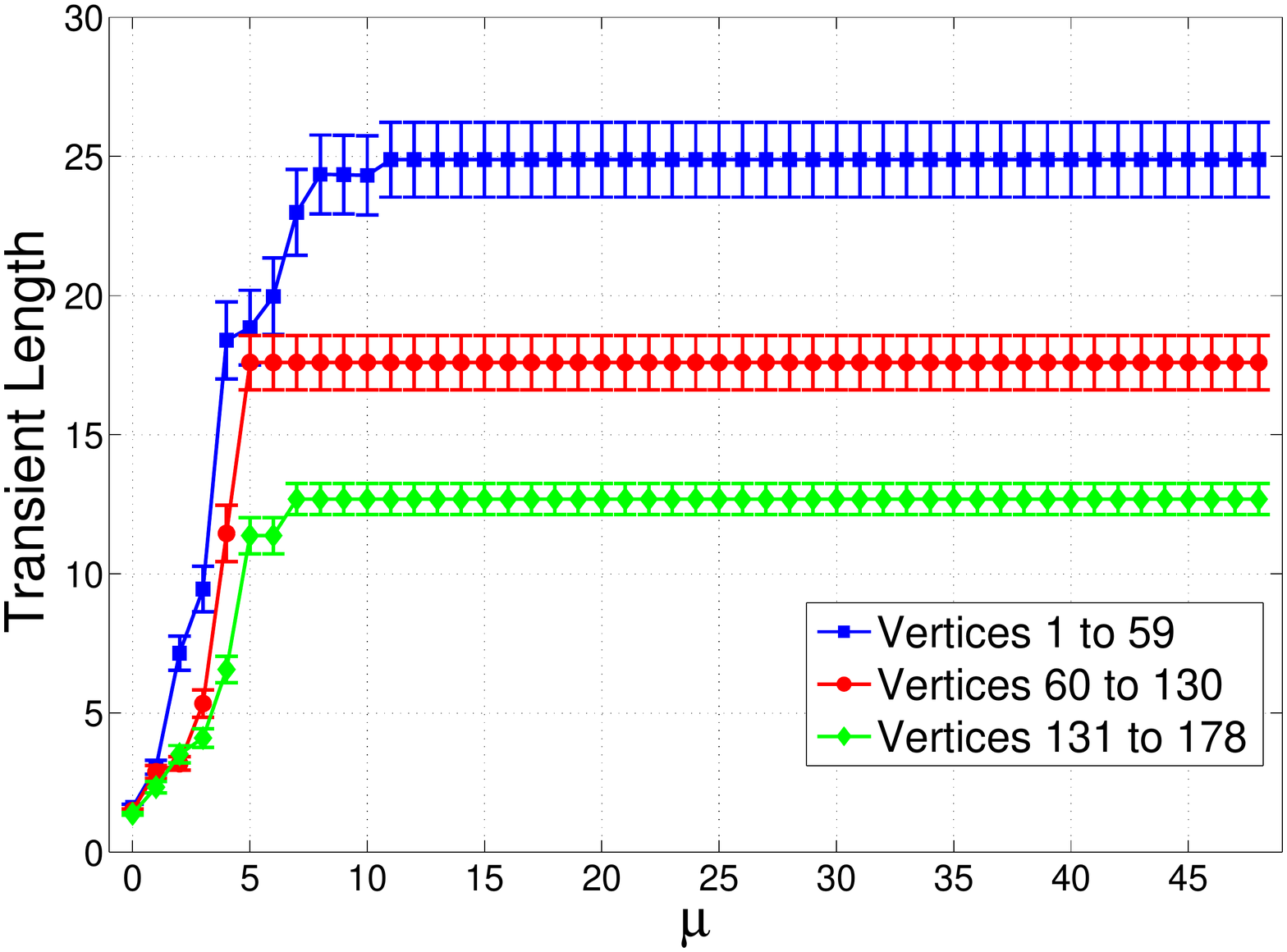}\label{fig:Wine-transient-length}}%
    \subfloat[]
    {\includegraphics[scale = 0.16]{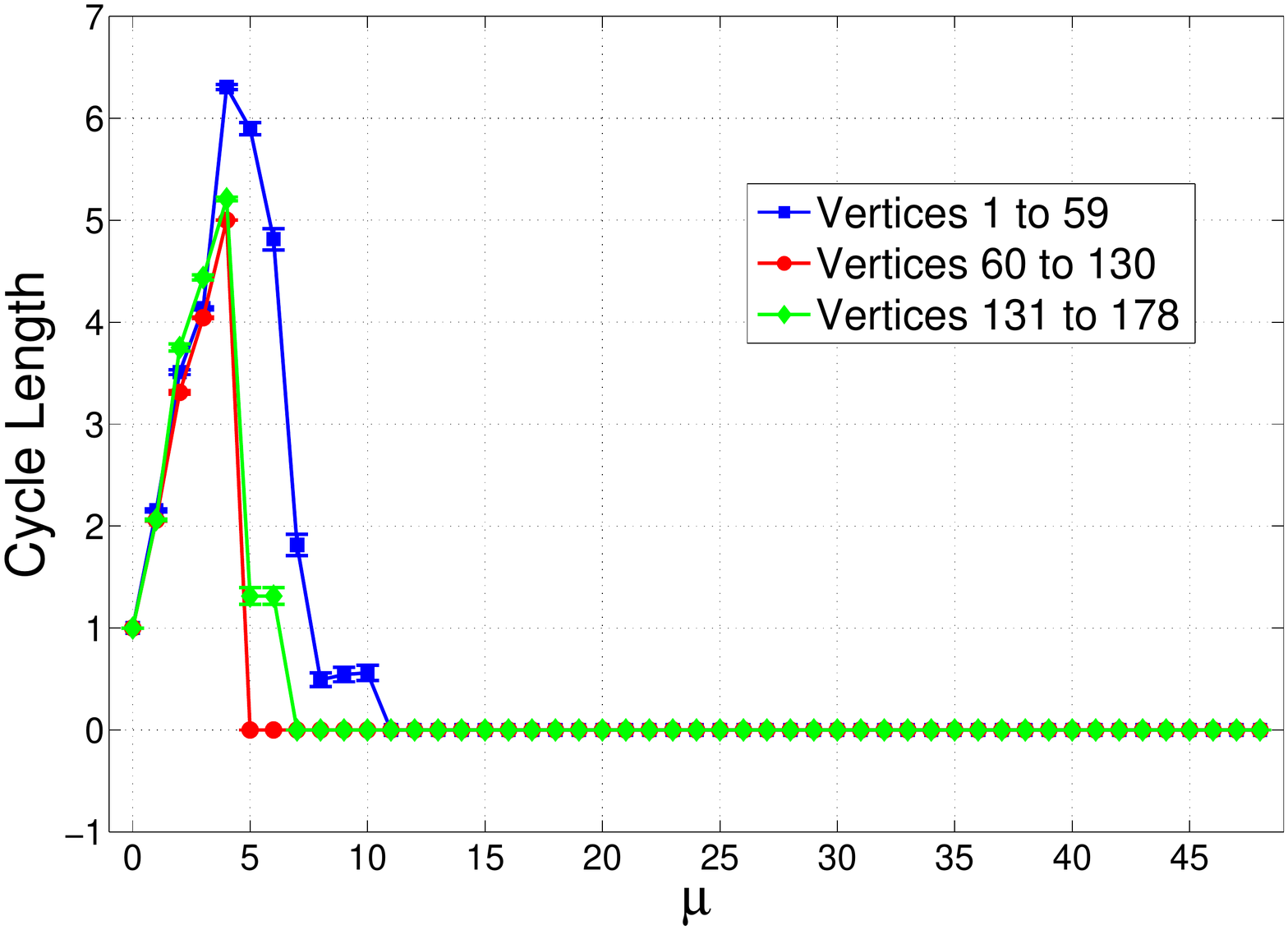}\label{fig:Wine-cycle-length}}%

    \caption{Behavior of the transient and cycle lengths of the Wine data set. Network's parameters: $k = 1$ and $\epsilon = 0.03$. (a) Transient lengths vs. $\mu_{c}$. (b) Cycle lengths vs. $\mu_{c}$.
    }%
    \label{fig:iris-transient-cycle}
\end{figure}

Continuing our exploration of this interesting phenomenon, we now turn our attention to two well-known data sets from the UCI Machine Learning Repository \cite{UCI2010}: Iris (balanced classes) and Wine (unbalanced classes). Consider Figs. \ref{fig:Iris-transient-length} and \ref{fig:Iris-cycle-length}, where it is depicted the transient and cycle lengths for the classes of the Iris data set, and Figs. \ref{fig:Wine-transient-length} and \ref{fig:Wine-cycle-length}, in which it is displayed the same information for the Wine data set. With respect to the transient length behavior, we can see that, for both data sets, as $\mu$ increases, the transient length also increases. However, when $\mu$ is sufficiently large,  the components' transient lengths settle down in a flat region (like the previous case). On the other hand, for both data sets, the cycle length behavior is rather interesting, which can be roughly divided in three different regions: (i) for a small $\mu$, it is directly proportional to $\mu$; (ii) for intermediate values of $\mu$, it is inversely proportional to $\mu$; and (iii) for sufficiently large values of $\mu$, it also settles down in a steady region. One can interpret these results as follows:

\begin{itemize}
  \item When $\mu$ is small, it is very likely that the transient and cycle parts will also be small, because the memory of the tourist is very limited. We can conceive this as a walk with almost no restrictions;

  \item When $\mu$ assumes an intermediate value, the transient length keeps increasing but the cycle length reaches a peak and starts to decrease afterwards. This peak characterizes the topological complexity of the component and varies from one to another. Hence, this is the most important region for capturing pattern formation of the class component by using the topological structure of the network.

  \item When $\mu$ is large, the tourist has a greater chance of getting trapped in a vertex of the graph, once all the neighborhood of the visited vertex is contained in the memory window $\mu$. In this scenario, the transient length is expected to be very high and the cycle length, null. This phenomenon explains the steady regions in Figs. \ref{fig:Iris-transient-length} and \ref{fig:Iris-cycle-length}. In this region, we can say that the tourist walks have already covered all the global aspects of the class component, and increasing the memory length $\mu$ will not capture any new topological features or pattern formation of the class. In this scenario, it is said that the tourist walks have completely described the topological complexity of the class component (saturation). In view of this, the calculation of tourist walks by further increasing $\mu$ is redundant.

\end{itemize}

This analysis suggests that the accuracy of the high level classifier may not change given that we choose a $\mu_{c}$ residing near these steady regions. This means that higher values for $\mu_{c}$ will only cause redundant computations and the accuracy will not be enhanced.   In order to check this, Figs. \ref{fig:Iris-Accuracy} and \ref{fig:Wine-Accuracy} reveal the behavior of the hybrid classification framework for different values of $\mu_{c}$. We have used three distinct compliance terms, namely $\lambda \in \{0, 0.05, 0.6\}$ for the Iris data set and $\lambda \in \{0, 0.06, 0.7\}$ for the Wine data set. When $\lambda =0$, only the low level classifier is used, in such a way that the value of $\mu_{c}$ is irrelevant, since the high level term is disabled. With respect to the other cases, when $\mu_{c} \ge 20$, the model provides the same accuracy rates for both data sets, confirming our prediction that, once it reaches the steady region where the transient and cycle lengths settle down, subsequent increases of $\mu_{c}$ do not change the accuracy of the model. In practical terms, according to our simulations, fixing the critical length $\mu_{c}$ as $10 - 30\%$ of the component's size is enough to get satisfactory results.

\begin{figure} [!htb]
    \centering
    \subfloat[]
    {\includegraphics[scale = 0.16]{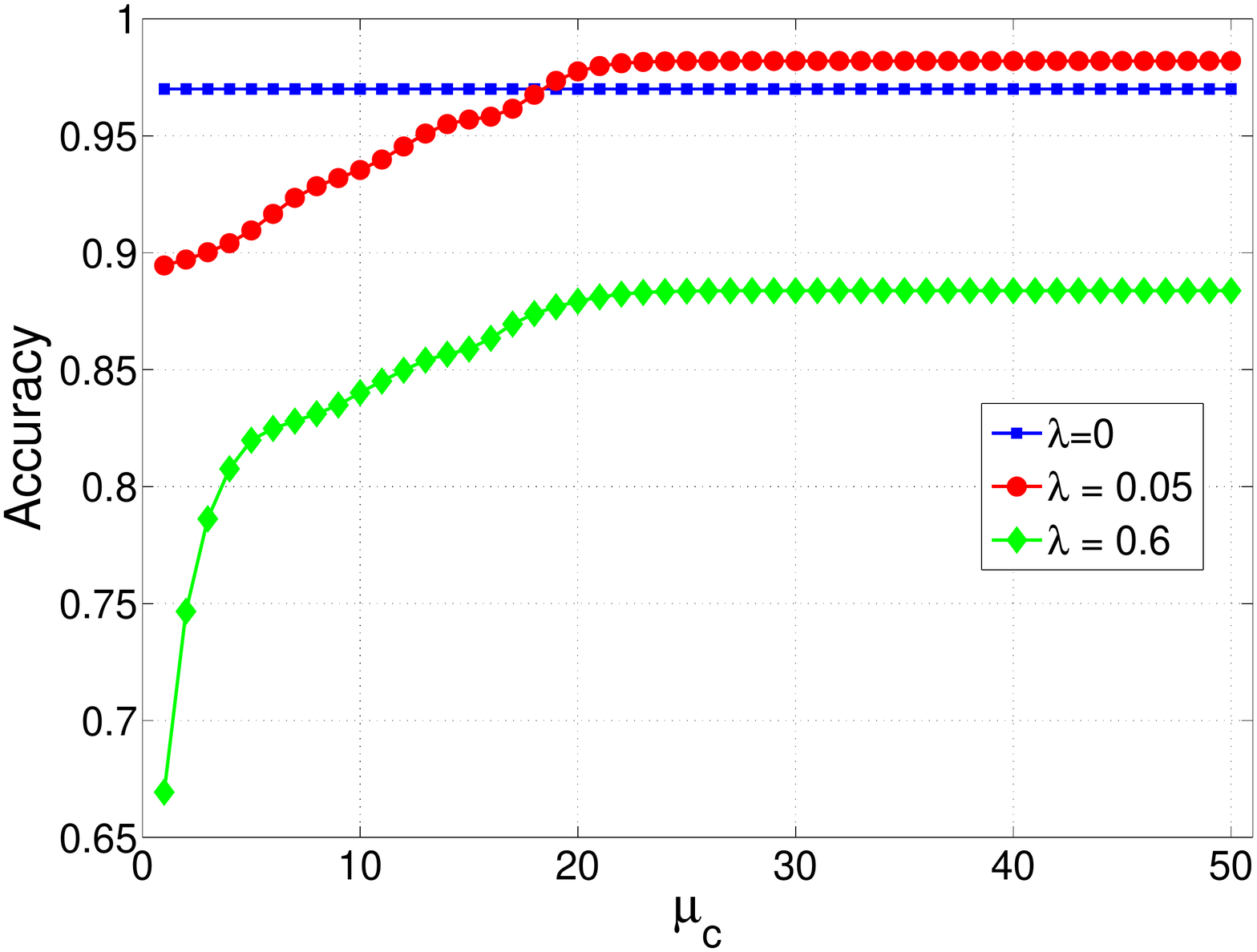}\label{fig:Iris-Accuracy}}%
    \subfloat[]
    {\includegraphics[scale = 0.16]{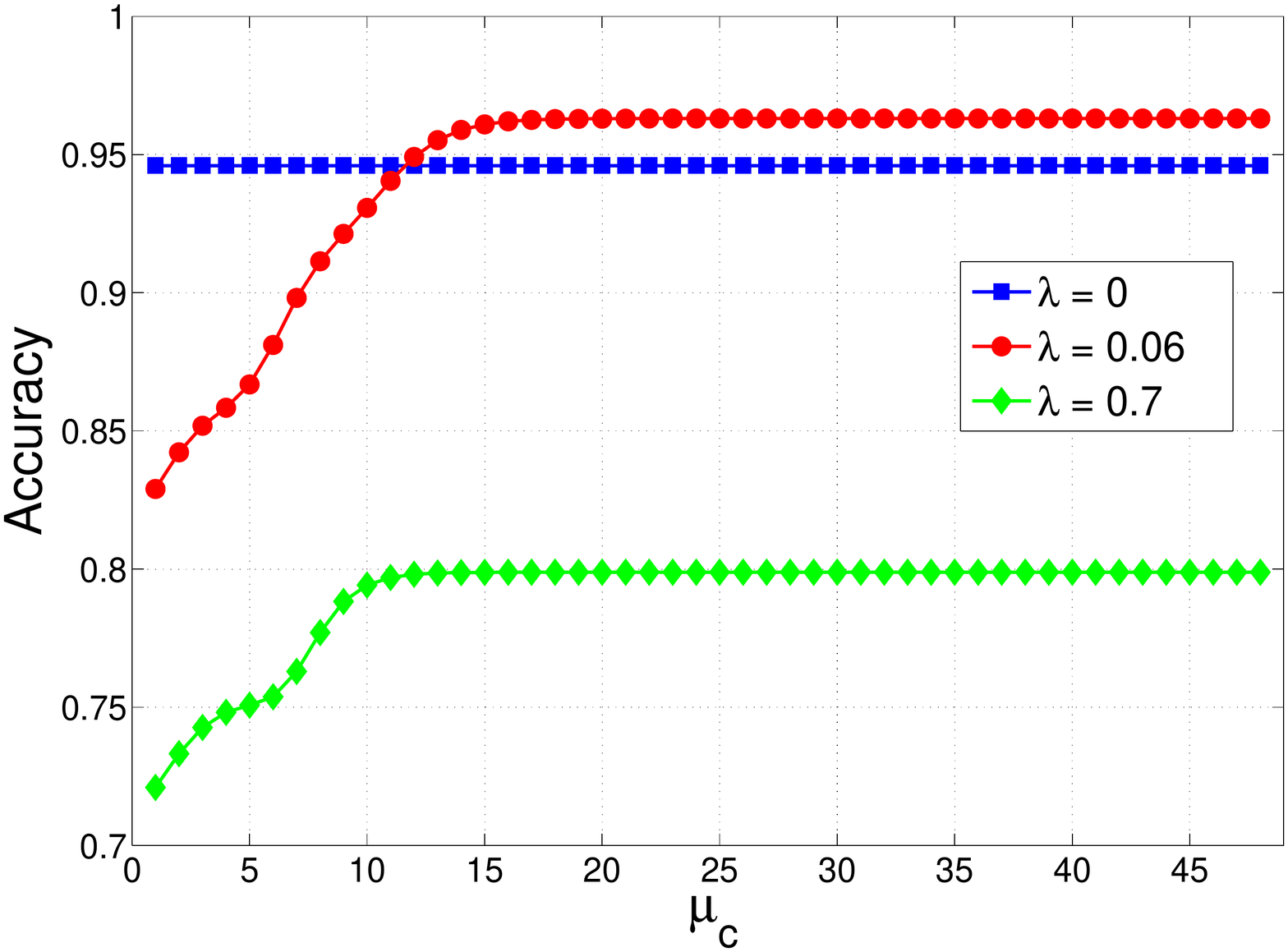}\label{fig:Wine-Accuracy}}%

    \caption{Accuracy rate vs. $\mu_{c}$ for three different values of the compliance term $\lambda$.
    The mean of a stratified $10$-fold cross-validation is reported.
    (a) Iris data set ($\alpha_{c} = 0.6$, $\alpha_{t} = 0.4$, low level classifier: fuzzy SVM with RBF kernel - $C=2^{-2}$ and $\gamma = 2^{3}$) and
    (b) Wine data set ($\alpha_{c} = 0.6$, $\alpha_{t} = 0.4$, low level classifier: weighted $k$-NN - $k = 5$).}%
    \label{fig:tourist-accuracy-behavior}
\end{figure}

\subsection{Simulations on Real-World Data Sets}
\label{sec:real-sim}

In this section, we will apply the proposed framework to several well-known UCI data sets. The most relevant metadata of each data set is given in Table \ref{tab:uci-datasets}. For a detailed description, refer to \cite{UCI2010}. Concerning the numerical attributes, the reciprocal of the Euclidean distance is employed. For categorical examples, the overlap similarity measure \cite{Boriah2008} is utilized. All data sets are submitted to a standardization pre-processing step.

\begin{table}[!htbp]
  \centering
  \small\addtolength{\tabcolsep}{-3.5pt}
  \caption{Brief information of the data sets.}
    \begin{tabular}{cccccc}
    \addlinespace
    \toprule
     \textbf{Data Set}  & \textbf{\# Samples} & \textbf{\# Dimensions} & \textbf{\# Classes} & $\mathbf{\alpha_{t}}$ & $\mathbf{\alpha_{c}}$ \\

    \midrule
    \textbf{Yeast} & 1484   & 8      & 10 & 0.40   & 0.60\\
    \textbf{Teaching} & 151    & 5      & 3 & 0.50   & 0.50\\
    \textbf{Zoo} & 101    & 16     & 7 & 0.30   & 0.70\\
    \textbf{Wine} & 178    & 13     & 3 & 0.40   & 0.60\\
    \textbf{Iris} & 150    & 4      & 3 & 0.40   & 0.60\\
    \textbf{Glass} & 214    & 9      & 6 & 0.50   & 0.50\\
    \textbf{Vehicle} & 846    & 18     & 4 & 0.60   & 0.70\\
    \textbf{Letter} & 20000  & 16     & 26 & 0.40   & 0.60\\
    \bottomrule
    \end{tabular}%
  \label{tab:uci-datasets}%
\end{table}%

Here, the high level classifier is composed of the best weighted combination of transient and cycle lengths. The optimization process is done by encountering $\alpha_{t} \times \alpha_{c} \in \{0, 0.1,\hdots, 1\} \times \{0, 0.1,\hdots, 1\}$ (search space), subjected to $\alpha_{t} + \alpha_{c} = 1$, which result in the highest accuracy rate of the model. The critical memory length is fixed to $\mu_{c}=0.3n_{\mathrm{max}}$, where $n_{\mathrm{max}}$ indicates the size of the largest component. The parameter optimization results are given in last two columns of Table \ref{tab:uci-datasets}.

Here, we will deal with two kinds of high level classifiers: (i) one in which the tourist walks are performed in a \emph{network} constructed from the vector-based data set and (ii) one in which the tourist walks are realized in a \emph{lattice}, i.e., the tourist is free to visit any other data site apart from the ones in the memory window $\mu$. In the latter case, it is expected that the walker will perform long jumps in the data set once the memory length $\mu$ assumes large values. It will be verified that such mechanism is not very welcomed in classification tasks. Furthermore, this serves as a strong argument for the employment and introduction of network-based high level classifiers.

%
%
%
%

Table \ref{tab:uci-results} reports the results obtained by the proposed technique on the data sets listed in Table \ref{tab:uci-datasets}. For comparison purposes, we evaluate the performance of the framework against different low level classifiers: Bayesian networks \cite{Neapolitan2003}, Weighted $k$-nearest neighbors \cite{Hastie2009}, Multi-layer perceptrons (MLP) \cite{Rumelhart1988}, Multi-class SVM (M-SVM) \cite{Lin2002,Abe2002}. The outcome of each algorithm is estimated by the average value over hundred runs of a stratified $10$-fold cross-validation process. Also, for each result, three different types of results are indicated as follows:

\begin{list}{\labelitemi}{\leftmargin=1em}
    \item ``Pure'' row: only the low level classifier is utilized ($\lambda = 0$). In this case, inside the parentheses are indicated the best parameters obtained from the optimization process for each technique;

    \item ``Networkless'' row: a mixture of low and high level classifiers is employed. The value inside the parentheses indicates the best compliance term $\lambda$. The high level classifier is constructed using the best weighted combination of transient and cycle lengths with the weights respecting Table \ref{tab:uci-datasets}. Moreover, the tourist walks are performed in a networkless environment, i.e., the tourist can visit any other site (data item) apart from those contained in the memory window $\mu$.

    \item ``Network'' row: the same setup as before, but the tourist walks are conducted on a networked environment. In this case, the tourist can only visit vertices (items) that are in the neighborhood and not in the memory window $\mu$. Here, the network in the training phase is built using $k = 1$ and $\epsilon = 0.03$. The values inside the parentheses exhibit: the $\epsilon$ used in the classification phase and the best compliance term $\lambda$, respectively.
\end{list}

For the sake of clarity, take the first entry of Table \ref{tab:uci-results}. The pure low level classifier Bayesian Networks achieved an accuracy rate of $57.8 \pm 2.6$ ($\lambda = 0$). However, if we use the proposed technique in a networkless environment, the accuracy rate is refined, achieving $58.6 \pm 2.3$ when $\lambda = 0.04$. Now, when the proposed technique is used in a network environment, the accuracy rate reaches $66.3 \pm 2.6$ when $\lambda = 0.28$. In general, the proposed technique is able to boost the accuracy rates of the data sets under analysis. Furthermore, we can see that the networked high level classifier can outperform the networkless version.

\begin{table*}[!htbp]
  \addtolength{\tabcolsep}{-4pt}
  \centering
  \caption{Accuracy rate achieved by several low level classification techniques and the high level classifier with and without networks. In the row named ``Pure'', the optimized parameters of each low level technique are reported as follows: Weighted $k$-NN ($k$), MLP (number of layers, learning rate, momentum), and fuzzy M-SVM ($C$, $\gamma$). In the row denominated ``Networkless'', the best compliance term is reported in the pair of parentheses. In the row called ``Network'', the $\epsilon$ (classification phase) and the best compliance term are exhibited in the pair of parentheses.}

    \begin{tabular}{cccccc}
    \addlinespace
    \toprule
     \textbf{Data Set}      &   \textbf{Type}     & \textbf{Bayesian Networks} & \textbf{Weighted k-NN} & \textbf{MLP} & \textbf{Fuzzy M-SVM} \\
    \bottomrule

    \multicolumn{1}{c}{\multirow{3}[0]{*}{\textbf{Yeast}}} & \textbf{Pure} & $57.8 \pm 2.6$   & $60.9 \pm 3.6$ ($16$)  & $56.2 \pm 3.9$ ($4, 0.3, 0.2$)  & $58.9 \pm 4.8$ ($2^{11}, 2^{0}$) \\

    \multicolumn{1}{c}{} & \textbf{Networkless} & $58.6 \pm 2.3$ ($0.04$)  & $62.0 \pm 3.2$ ($0.07$) & $56.9 \pm 2.5$ ($0.05$)  & $60.2 \pm 4.6$ ($0.14$)\\

    \multicolumn{1}{c}{} & \textbf{Network} & $66.3 \pm 2.6$ ($0.05$, $0.28$)  & $65.7 \pm 4.0$ ($0.03$, $0.19$)  & $63.3 \pm 2.9$ ($0.05$, $0.23$)  & $69.8 \pm 3.8$ ($0.05$, $0.28$)\\

    \midrule

    \multicolumn{1}{c}{\multirow{3}[0]{*}{\textbf{Teaching}}} & \textbf{Pure} & $61.3 \pm 8.8$   & $63.0 \pm 12.3$ ($9$)  & $60.9 \pm 9.4$ ($7, 0.2, 0.4$)  & $52.5 \pm 7.9$ ($2^{6}, 2^{3}$) \\

    \multicolumn{1}{c}{} & \textbf{Networkless} & $63.5 \pm 9.3$ ($0.18$)  & $63.8 \pm 10.6$ ($0.12$)  & $62.0 \pm 7.7$ ($0.13$)  & $55.3 \pm 8.6$ ($0.18$)\\

    \multicolumn{1}{c}{} & \textbf{Network} & $67.2 \pm 6.6$ ($0.03$, $0.24$)  & $68.5 \pm 7.4$ ($0.04$, $0.19$)  & $67.8 \pm 6.1$ ($0.04$, $0.26$)  & $62.7 \pm 6.9$ ($0.04$, $0.33$)\\

    \midrule

    \multicolumn{1}{c}{\multirow{3}[0]{*}{\textbf{Zoo}}} & \textbf{Pure} & $95.9 \pm 4.3$   & $96.2 \pm 5.8$ ($1$)  & $96.1 \pm 6.9$ ($3, 0.4, 0.5$)  & $96.3 \pm 6.4$ ($2^{1}, 2^{1}$) \\

    \multicolumn{1}{c}{} & \textbf{Networkless} & $96.0 \pm 3.6$ ($0.02$)  & $96.5 \pm 5.2$ ($0.04$)  & $96.4 \pm 6.6$ ($0.04$)  & $96.5 \pm 4.5$ ($0.05$)\\

    \multicolumn{1}{c}{} & \textbf{Network} & $97.3 \pm 4.3$ ($0.02$, $0.06$)  & $97.5 \pm 4.4$ ($0.02$, $0.09$)  & $97.5 \pm 4.2$ ($0.02$, $0.09$)  & $97.5 \pm 2.3$ ($0.02$, $0.08$)\\

    \midrule

    \multicolumn{1}{c}{\multirow{3}[0]{*}{\textbf{Wine}}} & \textbf{Pure} & $98.8 \pm 0.7$   & $94.6 \pm 1.4$ ($1$)  & $97.8 \pm 0.5$ ($3, 0.6, 0.4$)  & $98.9 \pm 0.2$ ($2^{11}, 2^{2}$) \\

    \multicolumn{1}{c}{} & \textbf{Networkless} & $98.8 \pm 0.7$ ($0.00$)  & $94.6 \pm 2.1$ ($0.00$)  & $97.9 \pm 0.3$ ($0.03$)  & $98.9 \pm 0.2$ ($0.00$)\\

    \multicolumn{1}{c}{} & \textbf{Network} & $98.8 \pm 0.7$ (-, $0.00$)  & $96.3 \pm 1.0$ ($0.06$, $0.06$)  & $98.6 \pm 0.3$ ($0.02$, $0.09$)  & $98.9 \pm 0.2$ (-, $0.00$)\\

    \midrule

    \multicolumn{1}{c}{\multirow{3}[0]{*}{\textbf{Iris}}} & \textbf{Pure} & $92.7 \pm 1.2$   & $97.9 \pm 3.3$ ($19$) & $94.0 \pm 2.9$ ($1, 0.3, 0.2$)  & $97.0 \pm 4.6$ ($2^{-2}, 2^{3}$) \\

    \multicolumn{1}{c}{} & \textbf{Networkless} & $93.2 \pm 1.9$ ($0.07$)  & $97.9 \pm 3.3$ ($0.00$) & $94.8 \pm 2.8$ ($0.10$)  & $97.2 \pm 3.7$ ($0.09$)\\

    \multicolumn{1}{c}{} & \textbf{Network} & $94.9 \pm 0.4$ ($0.01$, $0.15$)  & $98.3 \pm 0.6$ ($0.01$, $0.05$)  & $96.5 \pm 1.1$ ($0.02$, $0.21$)  & $98.1 \pm 1.0$ ($0.02$, $0.05$)\\

    \midrule

    \multicolumn{1}{c}{\multirow{3}[0]{*}{\textbf{Glass}}} & \textbf{Pure} & $70.6 \pm 7.7$   & $71.8 \pm 9.0$ ($1$) & $67.3 \pm 5.0$ ($7, 0.1, 0.3$)  & $72.4 \pm 5.6$ ($2^{10}, 2^{4}$) \\

    \multicolumn{1}{c}{} & \textbf{Networkless} & $71.5 \pm 5.7$ ($0.14$)  & $72.7 \pm 7.1$ ($0.16$)  & $68.8 \pm 3.2$ ($0.22$)  & $73.3 \pm 3.9$ ($0.12$)\\

    \multicolumn{1}{c}{} & \textbf{Network} & $79.2 \pm 5.3$ ($0.03$, $0.32$)  & $79.7 \pm 5.0$ ($0.35$)  & $77.4 \pm 5.5$ ($0.02$, $0.30$)  & $80.1 \pm 4.3$ ($0.03$, $0.31$)\\

    \midrule

    \multicolumn{1}{c}{\multirow{3}[0]{*}{\textbf{Vehicle}}} & \textbf{Pure} & $68.1 \pm 3.8$   & $67.6 \pm 4.1$ ($5$)  & $69.0 \pm 4.4$ ($5, 0.3, 0.2$)  & $84.4 \pm 3.4$ ($2^{10}, 2^{3}$) \\

    \multicolumn{1}{c}{} & \textbf{Networkless} & $70.0 \pm 2.6$ ($0.19$)  & $69.4 \pm 2.5$ ($0.18$)  & $70.3 \pm 3.8$ ($0.13$)  & $84.4 \pm 3.4$ ($0.00$)\\

    \multicolumn{1}{c}{} & \textbf{Network} & $74.1 \pm 2.9$ ($0.05$, $0.26$)  & $73.6 \pm 3.0$ ($0.05$, $0.24$)  & $74.7 \pm 3.6$ ($0.07$, $0.29$)  & $84.9 \pm 2.7$ ($0.04$, $0.07$)\\

    \midrule

    \multicolumn{1}{c}{\multirow{3}[0]{*}{\textbf{Letter}}} & \textbf{Pure} & $74.4 \pm 5.6$   & $96.0 \pm 7.6$ ($1$)  & $88.9 \pm 9.9$ ($3, 0.2, 0.4$)  & $94.8 \pm 1.7$ ($2^{6}, 2^{4}$) \\

    \multicolumn{1}{c}{} & \textbf{Networkless} & $75.5 \pm 4.6$ ($0.14$)  & $96.0 \pm 7.6$ ($0.00$)  & $89.3 \pm 7.4$ ($0.09$)  & $94.8 \pm 1.7$ ($0.00$)\\

    \multicolumn{1}{c}{} & \textbf{Network} & $77.8 \pm 3.4$ ($0.04$, $0.17$)  & $96.0 \pm 7.6$ (-, $0.00$)  & $92.1 \pm 4.1$ ($0.04$, $0.19$)  & $94.8 \pm 1.7$ (-, $0.00$)\\
    \bottomrule
    \end{tabular}%
  \label{tab:uci-results}%
\end{table*}%

\normalsize

\section{Application: Handwritten Digits Recognition}
\label{Simulation-Real}

In this section, we provide an application of the proposed technique to handwritten digits recognition. Specifically, in Subsect. \ref{Simulation-Real}, we detail the experimental results obtained from the modified NIST set, a data sets composed of tens of thousands of real handwritten digits. Still in this section, we present two simple examples to show how the proposed technique can be used for invariant pattern recognition.

While recognizing individual digits is only one of a myriad of problems that
involves specific designing of practical recognition system, it still is, undoubtedly,
an excellent benchmark for comparing shape recognition methods. The data set in
which we will conduct our studies hereon is named Modified NIST set \cite{Lecun1998}. This
data set provides a training set with $60 \ 000$ samples and a test set of $10 \ 000$
samples. Each image has $28 \times 28$ pixels.  With respect to the high level classifier, the network in the training phase is constructed using $k = 3$ and $\epsilon = 0.01$. In the classification phase, the same $\epsilon$ is used.

For comparison matters, we use the proposed high level classifier with $3$ distinct low level classifiers. The high level classifier will remain constant as we have been using, i.e., with an underlying network. The techniques that will be exploited are listed below:

\begin{itemize}
    \item A linear classifier implemented as an $1$-layer neural network. We use the same experimental setup given in \cite{Lecun1998};

    \item A $k$-nearest neighbor classifier with an Euclidean distance measure between input images ($k=3$);

    \item A $k$-nearest neighbor classifier with the similarity function given by a set of weighted eigenvalue measures \cite{Silva2012HighLevel}. Specifically, we use $\phi = 4$ eigenvalues and adopt the following $\beta$ function: $\beta(x) = 16\exp(\frac{x}{3})$.
\end{itemize}

Figure \ref{fig:MNIST-comparison} shows the performance of these techniques acting together with the high level classifier in a networked environment. Our main goal here is to reveal that a mixture of traditional and high level classifier is able to increase the accuracy rate. For example, the linear neural network reached $88\%$ of accuracy rate when only a traditional classifier  is applied. A small increase in the compliance term is responsible for increasing the accuracy rate to $91\%$ ($\lambda = 0.2$). Regarding the $k$-nearest neighbor algorithm, for a pure traditional classifier, we have obtained $95\%$ of accuracy rate, against $97.6\%$ ($\lambda = 0.25$). For the proposed weighted eigenvalue measure, we have obtained $98\%$ of accuracy rate when $\lambda = 0$, against $99.1\%$ when $\lambda = 0.2$. It is worth noting that the enhancement is significant. Even in the third case, the improvement is quite welcomed, because it is a hard task to increase an already very high accuracy rate. Moreover, one can see that maximum compliance term is intrinsic to the data set, since, for three completely distinct low level classifiers, the maximum accuracy rate is reached in the surroundings of $\lambda = 0.225$.

\begin{figure} [!htb]
    \centering
    \includegraphics[scale = \sizeOfFigure]{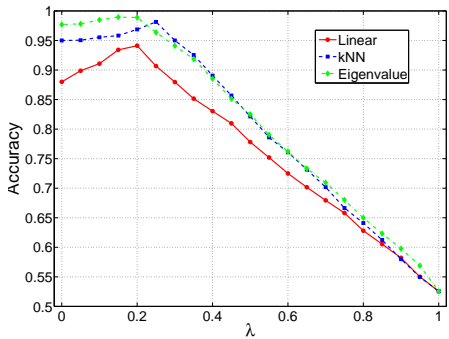}
    \caption{A detailed analysis of the impact of the compliance term $\lambda$ on different traditional low level techniques applied to the MNIST database. One can see that a mixture of the proposed traditional and high level techniques does give a boost in the accuracy rate in this real-world data set.
    }%
    \label{fig:MNIST-comparison}
\end{figure}

\begin{figure*} [!htb]
    \centering
    \subfloat[]
    {\includegraphics[scale = 0.36]{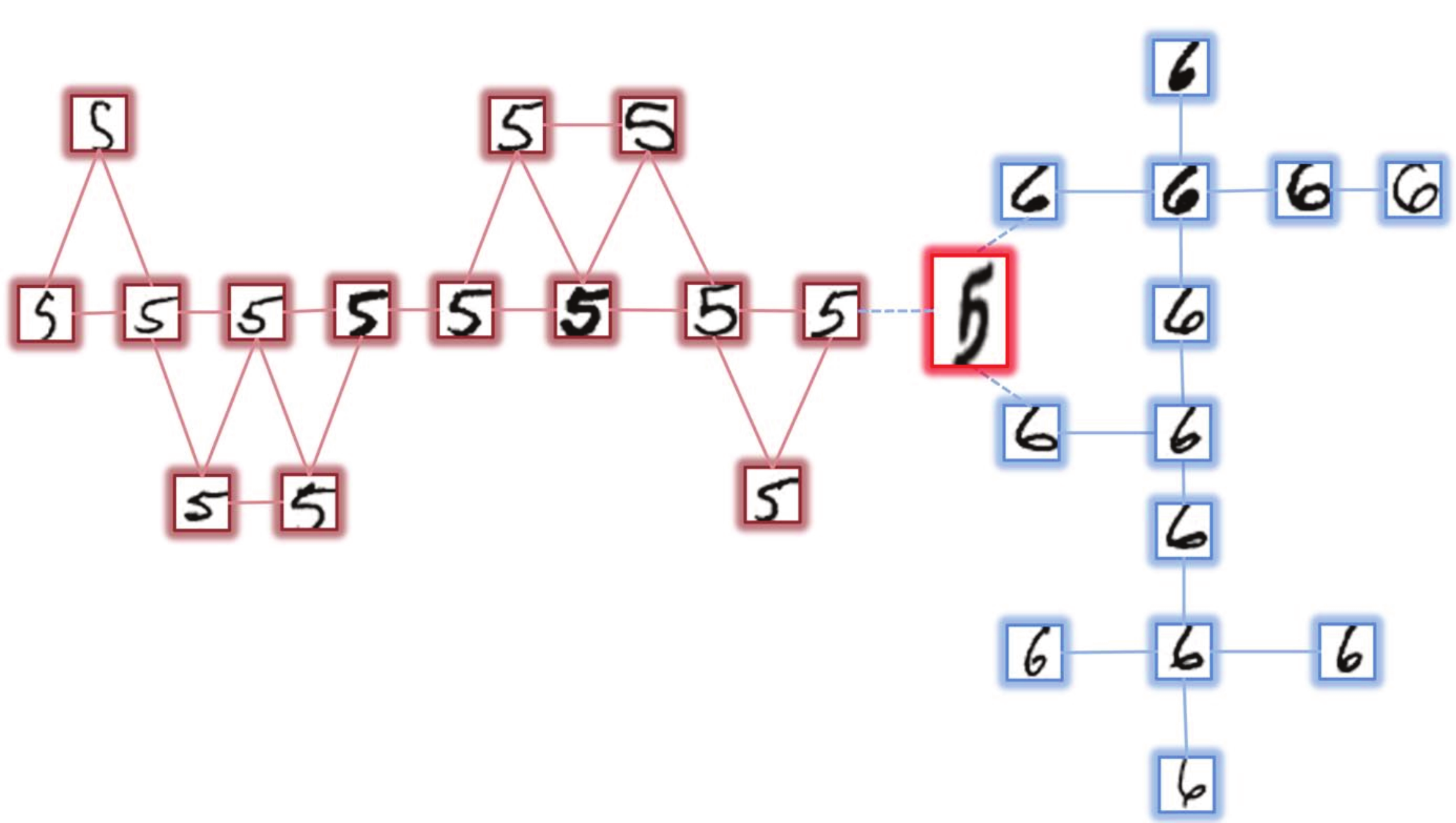}\label{fig:illustrative-network-1}}%
    \quad \subfloat[]
    {\includegraphics[scale = 0.36]{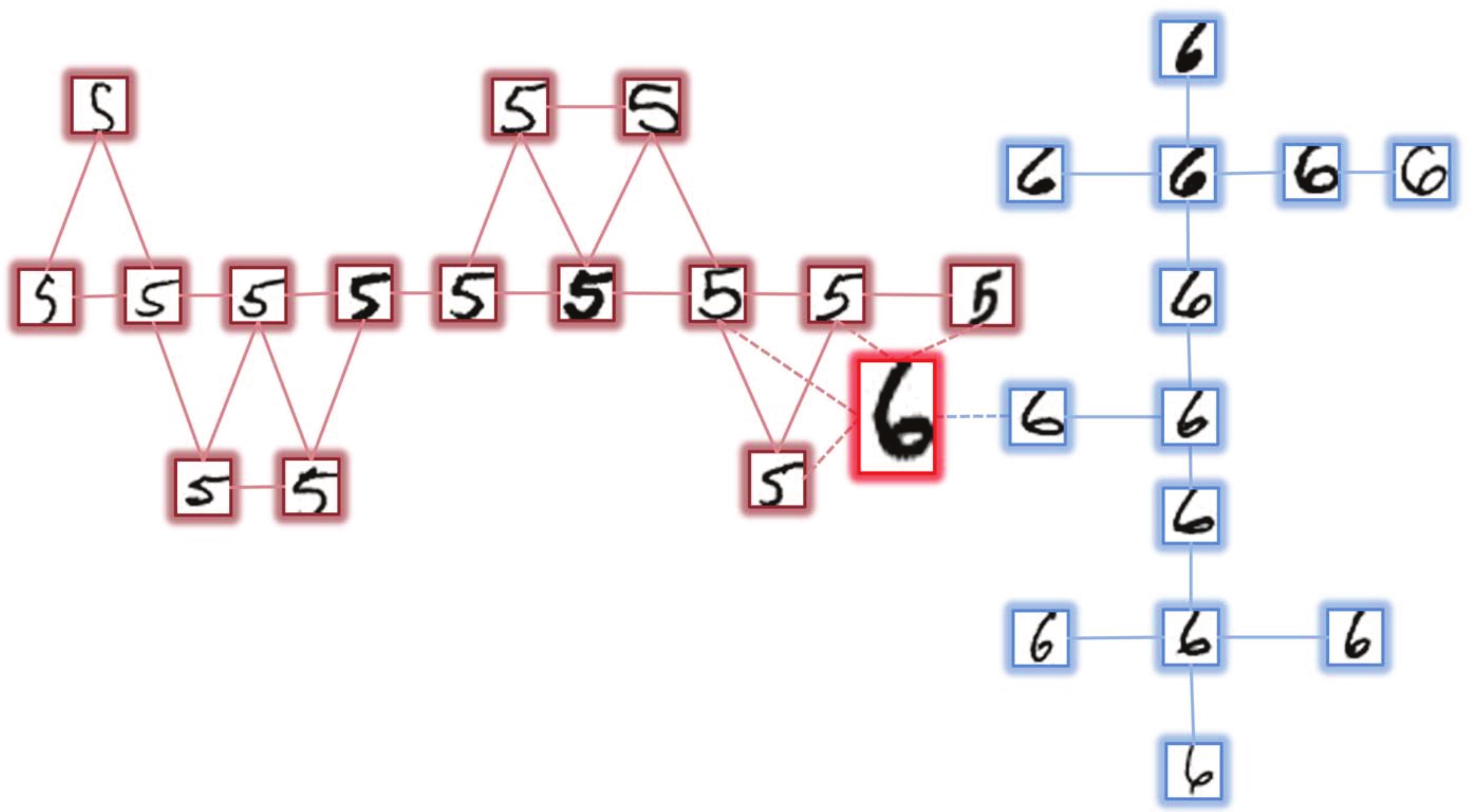}\label{fig:illustrative-network-2}}%

    \caption{Illustration of the pattern formation impact in a subset of samples extracted from the MNIST data set.
    The training instances are displayed by the brown (digit $5$) and blue (digit $6$) colors. The test instances are indicated
    by a red color (bigger sizes). Insertion of a test instance whose real class is: (a) the digit $5$ and (b) the digit $6$.}%
\end{figure*}

\begin{figure*} [!htb]
    \centering
    \subfloat[]
    {\includegraphics[scale = 0.16]{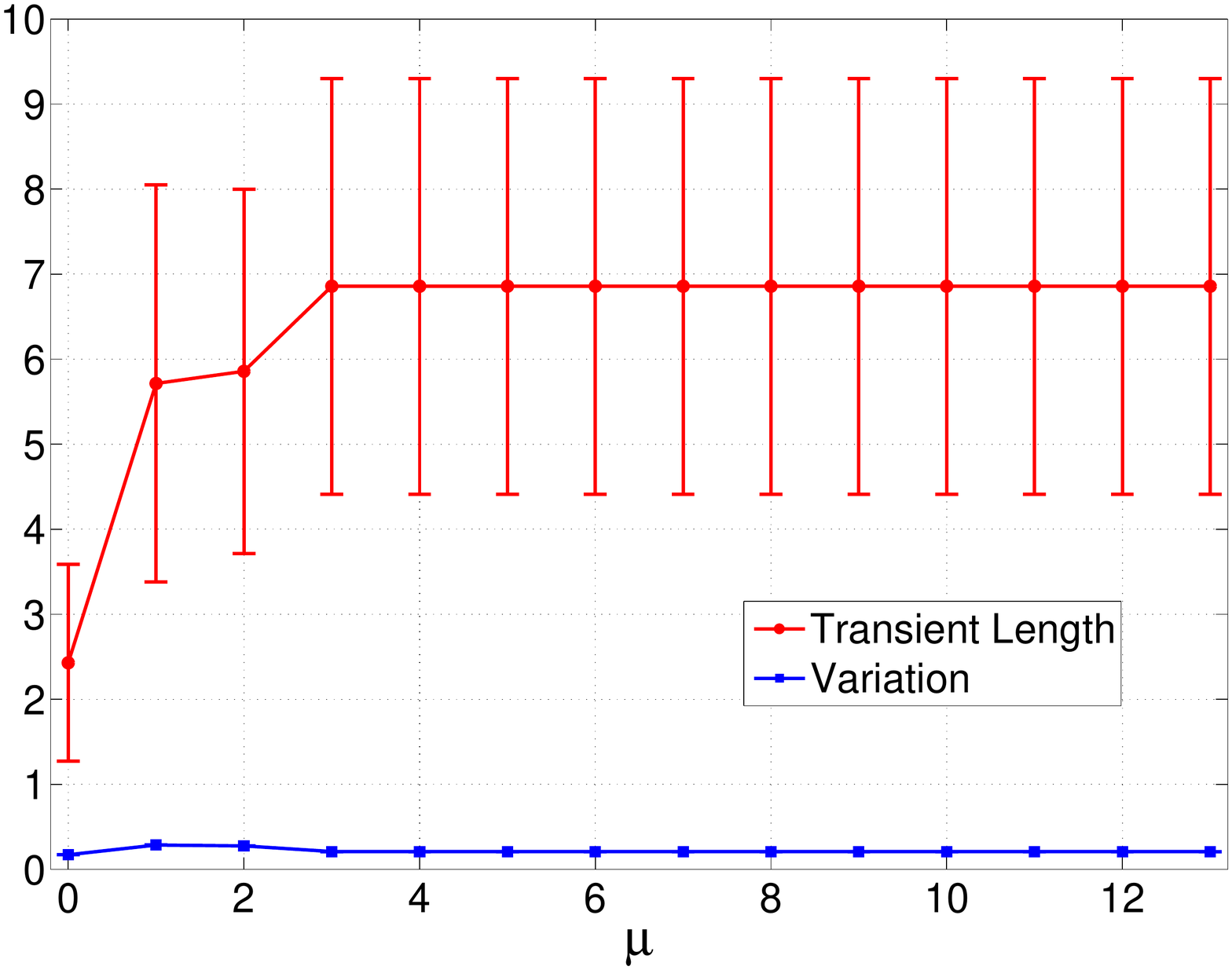}\label{fig:before-transientClass5}}%
    \subfloat[]
    {\includegraphics[scale = 0.16]{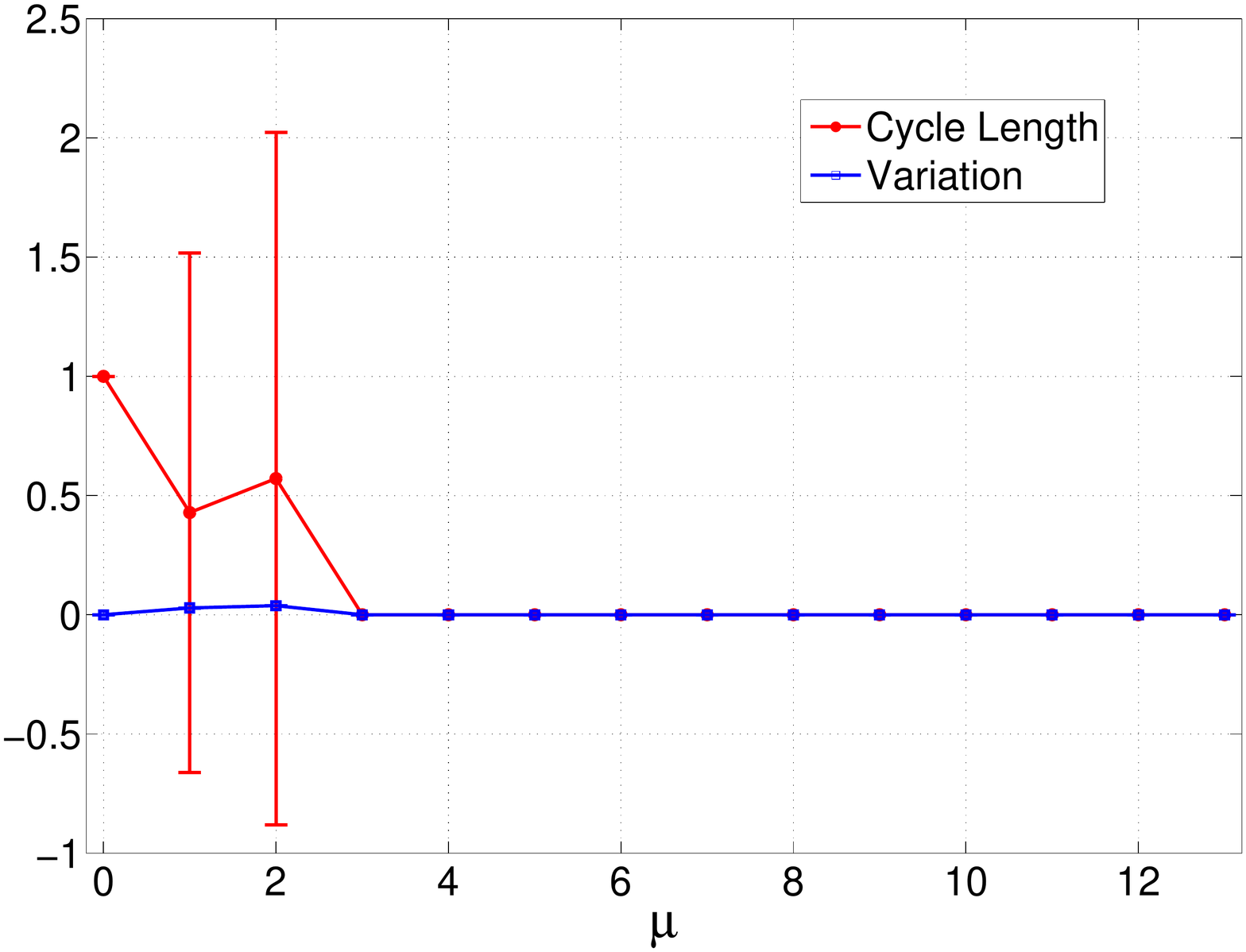}\label{fig:before-cycleClass5}}%
    \subfloat[]
    {\includegraphics[scale = 0.16]{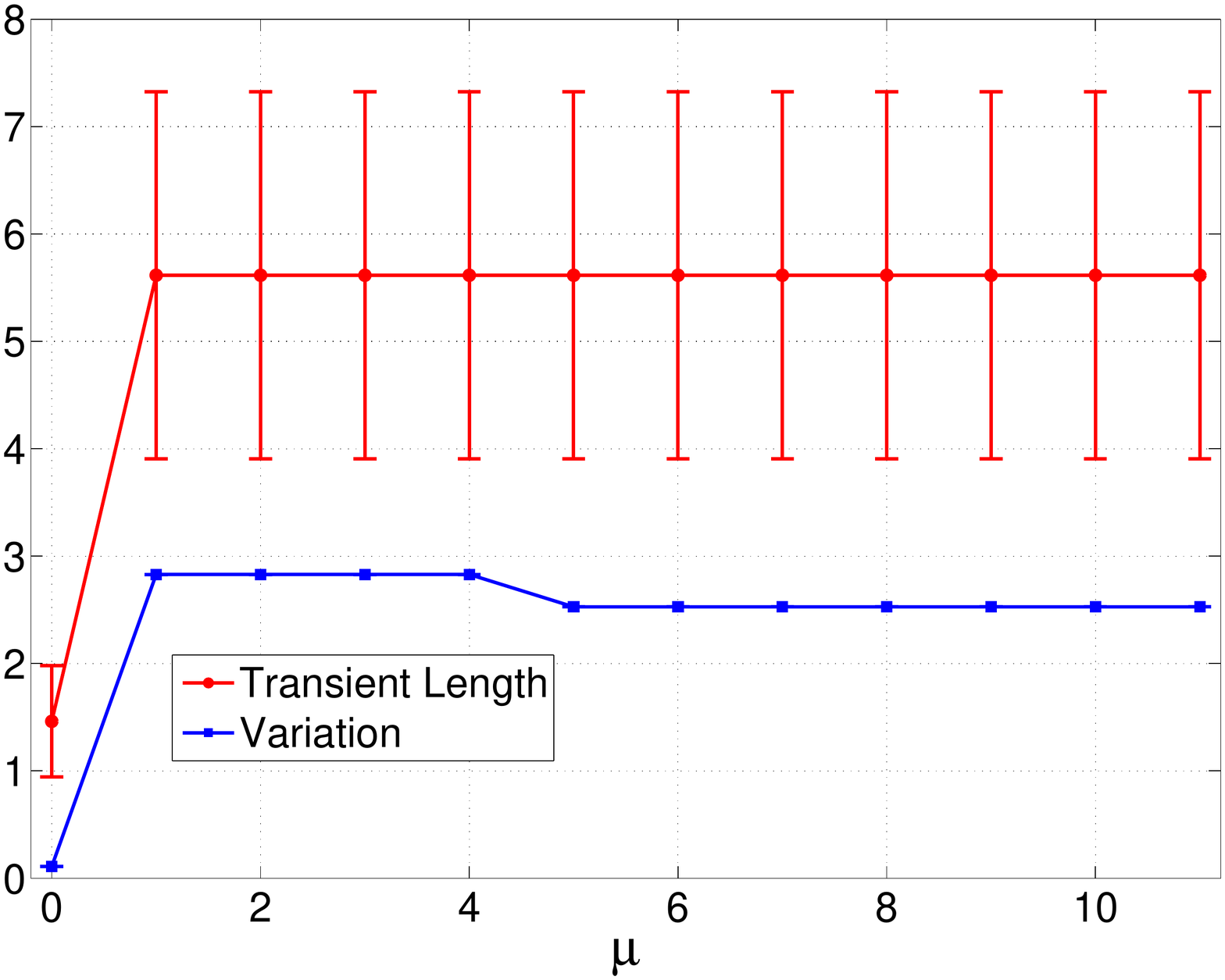}\label{fig:before-transientClass6}}%
    \subfloat[]
    {\includegraphics[scale = 0.16]{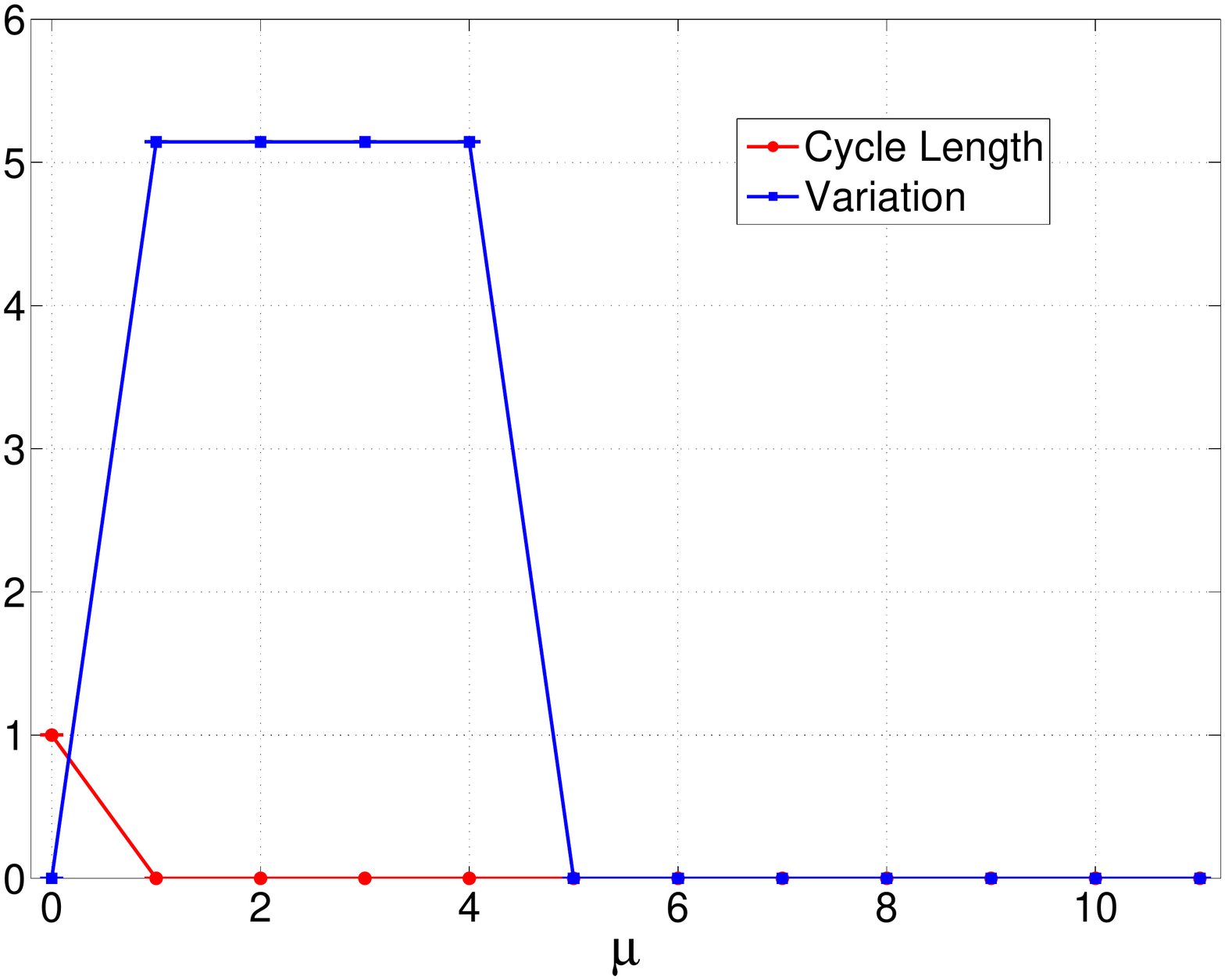}\label{fig:before-cycleClass6}}%

    \subfloat[]
    {\includegraphics[scale = 0.16]{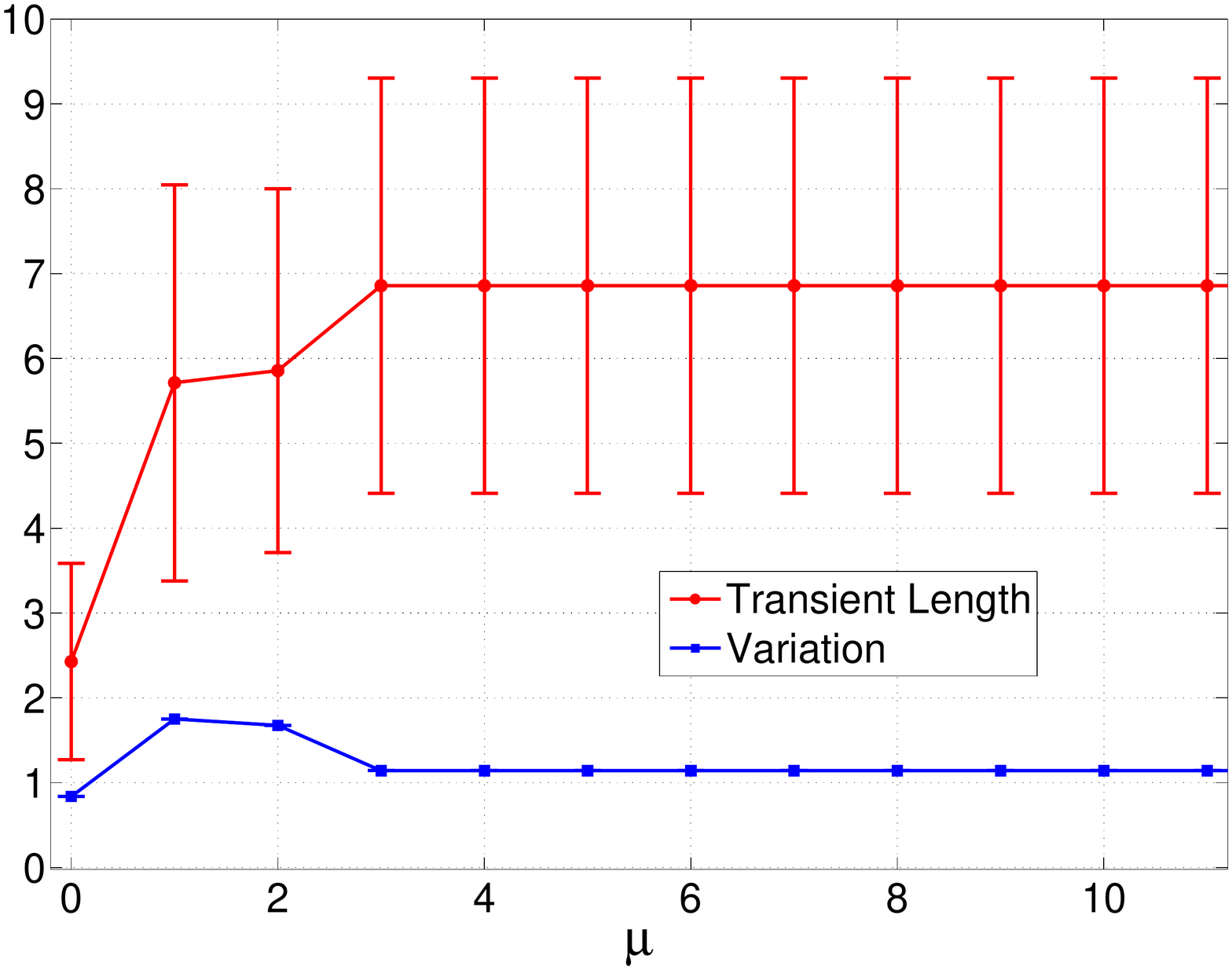}\label{fig:after-transientClass5}}%
    \subfloat[]
    {\includegraphics[scale = 0.16]{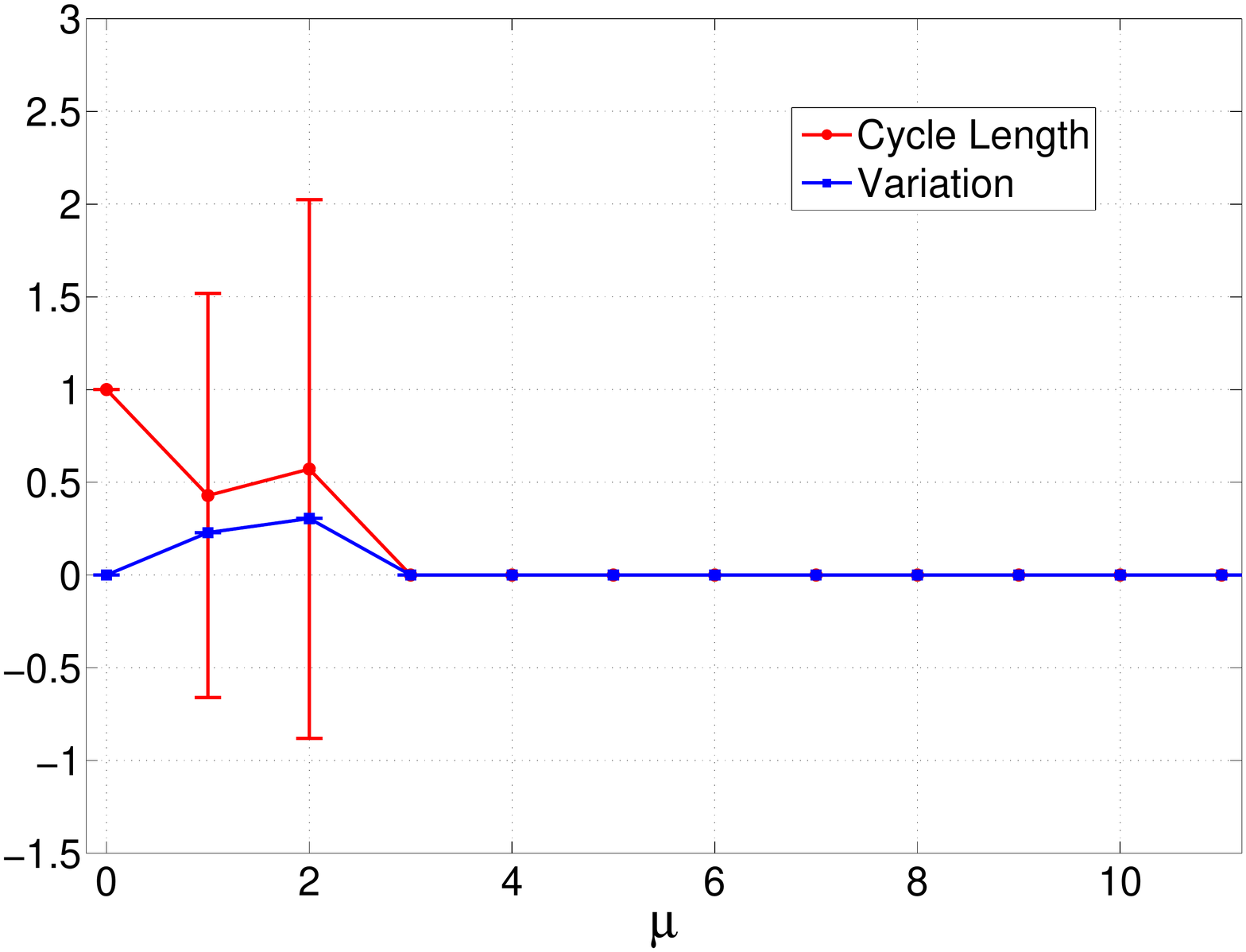}\label{fig:after-cycleClass5}}%
    \subfloat[]
    {\includegraphics[scale = 0.16]{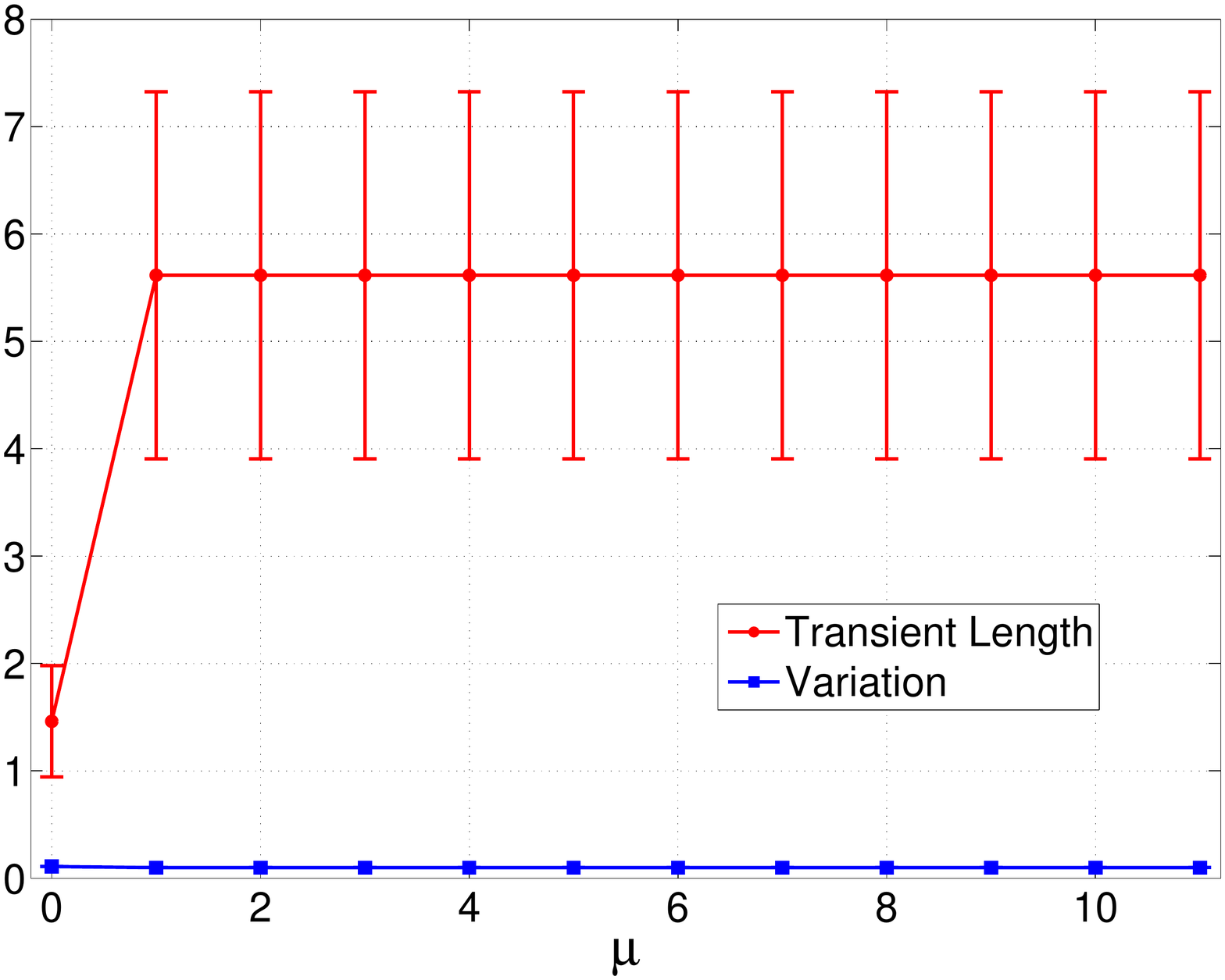}\label{fig:after-transientClass6}}%
    \subfloat[]
    {\includegraphics[scale = 0.16]{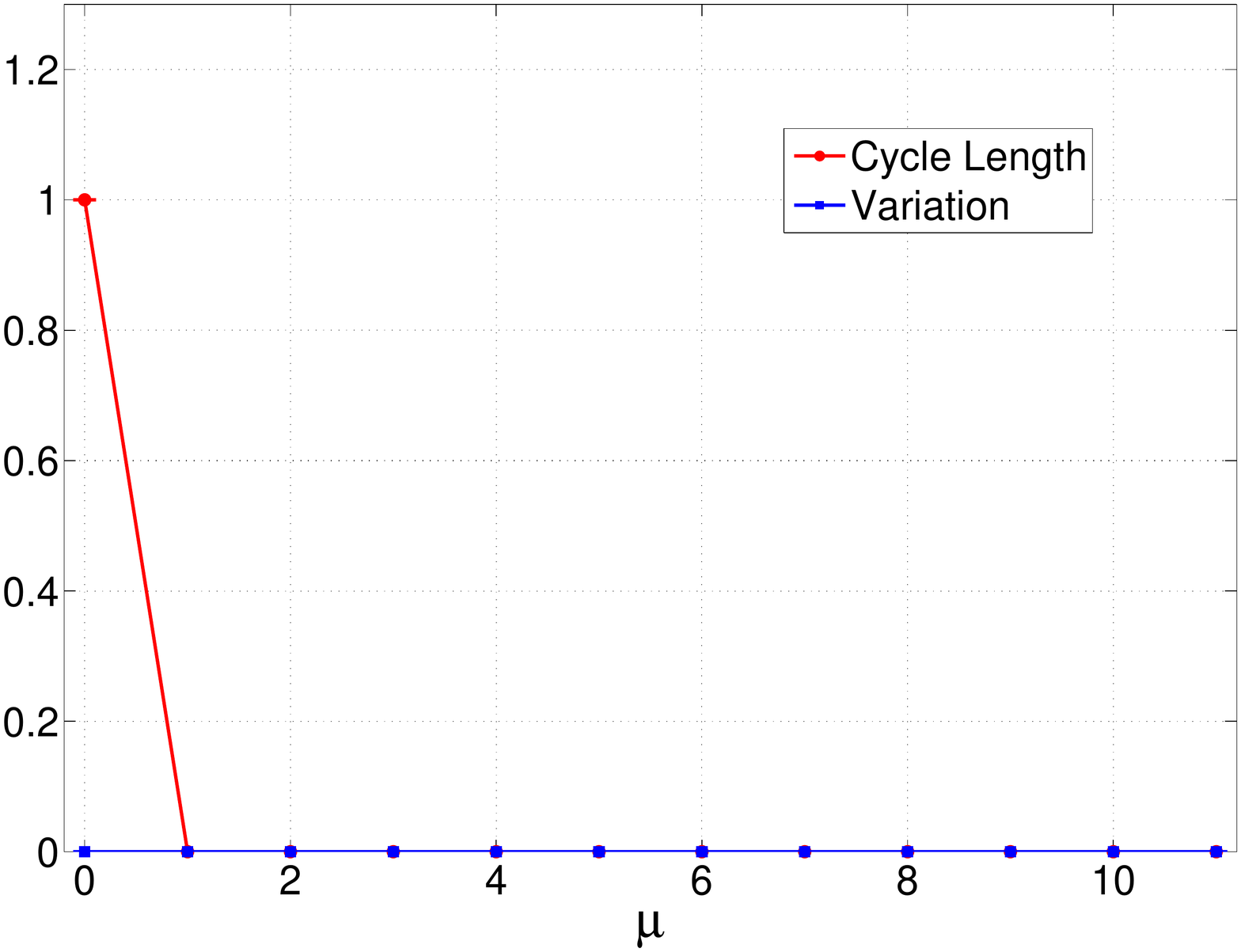}\label{fig:after-cycleClass6}}%

    \caption{Transient and cycle lengths and the corresponding variations due to the insertion of a test instance.
    With regard to Fig. \ref{fig:illustrative-network-1}: (a) and (b) transient and cycle lengths of the brown class (digit $5$) and their variations by virtue of the insertion of the red instance (test instance); (c) and (d) same information for the blue class (digit $6$). With regard to Fig. \ref{fig:illustrative-network-2}: (e) and (f) transient and cycle lengths of the brown class (digit $5$) and their corresponding variations due to the insertion of the red instance (test instance) with respect to the brown class; (g) and (h) same information for the blue class (digit $6$).
    }%
\end{figure*}

Figures \ref{fig:illustrative-network-1} and \ref{fig:illustrative-network-2} illustrate how the digit classification is carried out by using simple networks containing a small number of digits `5' and `6', randomly drawn from the MNIST data set. Firstly, let us consider the Fig. \ref{fig:illustrative-network-1}, where the digits `5' with brown boxes and the digits `6' with blue boxes represent the training set. The task now is to classify the test instance represented by the digit with a red box. If only low level classification is applied, the test digit is probably to be classified as a digit `6', because it has more neighbors of digit `6' than that of `5'. On the other hand, if we also consider high level classification, the test digit can be correctly classified as a digit `5', because it complies more to the pattern formed by training digits `5' than to the one formed by digits `6'. More specifically, if the test digit is put into the class `5', it just extends the already formed ``line" pattern. As a consequence, the inclusion of the test digit to the class `5' generates small variations of the component measures. However, if the test digit is inserted into the class `6', larger variations of the component measures occur, since cycles are formed in the component (before insertion of the test digit, there is no cycle in the component). Figures \ref{fig:before-transientClass5} and \ref{fig:before-cycleClass5} show the transient and cycle lengths as well as the corresponding variations as a function of $\mu$, when the test digit is inserted into the component of digit `5'. We see that the variations are very small, indicating the strong compliance of the test digit with the pattern formed by training digits `5'. On the other hand, Figs. \ref{fig:before-transientClass6} and \ref{fig:before-cycleClass6} show the same information, when the test digit is put into the component of digit `6'. Here, we see that larger variations occur, which means that the test digit does not conform to the pattern formed by the component of digits `6'. As a result, the test instance is correctly classified as a digit `5'. The same reasoning can be applied to the digit network shown by Fig. \ref{fig:illustrative-network-2}. In this case, the transient and cycle lengths as well as the corresponding variations are shown in Figs. \ref{fig:after-transientClass5} -  \ref{fig:after-cycleClass6}, when the test digit is inserted into the component of digit `5' or `6', respectively. In this situation, the test instance is classified as a digit `6'.

\section{Conclusions}
\label{sec:Conclusions}

In this work, we have proposed an alternative and novel technique for data classification, which combines both low and high level characteristics of the data. The former classifies data instances by their physical features and the latter measures the compliance of the test instance with the pattern formation of the input data. To this end, tourist walks have been employed to capture the complex topological properties of the network constructed from the input data. A quite interesting feature of the proposed technique is that the high level term's influence has to be increased in order to get correct classification as the complexity of the class configuration increased. This means that the high level term is specially useful in complex situations of classification. Also, it is worth observing that the application of the tourist walks dynamics in the context of high level classifier is also a novel approach in the literature. We have shown that, even though such walk is constructed under very simple rules, it is still able to capture topological features of the underlying network in a local to global basis. In addition, we uncover a critical memory length, above which no change occurs in the transient and cycle lengths of the network component.


We hope our work can provide an alternative way to the understanding of high level semantic machine learning. As a future work,  mechanisms for taking representative samples, rather than recalculating all the tourist walks for all the vertices in the network, are going to be considered.

\ifCLASSOPTIONcompsoc
  \section*{Acknowledgments}
\else
  \section*{Acknowledgment}
\fi

This work is supported by the São Paulo State Research Foundation
(FAPESP) and by the Brazilian National Research Council
(CNPq).

\ifCLASSOPTIONcaptionsoff
  \newpage
\fi



\bibliographystyle{IEEEtran}
\bibliography{IEEEabrv,TouristWalks_Christiano}
\end{document}